\documentclass{article}

\usepackage{microtype}

\usepackage{bm}

\usepackage[utf8]{inputenc}
\usepackage[disable]{todonotes}
\usepackage{amsmath}
\usepackage{amsthm}
\usepackage{amssymb}
\usepackage{bbm}
\usepackage{algorithm}
\usepackage{algorithmic}
\usepackage{booktabs}
\usepackage{multirow}
\usepackage{subcaption}
\usepackage{url}
\usepackage{pifont}
\usepackage{graphicx}

\usepackage{stfloats}
 
\newtheorem{theorem}{Theorem}
\newcommand{\xmark}{\ding{55}}

\usepackage[accepted]{icml2019}

\icmltitlerunning{Very Fast Streaming Submodular Function Maximization}

\begin{document}

\twocolumn[
    \icmltitle{
        Very Fast Streaming Submodular Function Maximization
    }
    
    \begin{icmlauthorlist}
    \icmlauthor{Sebastian Buschjäger}{tu}
    \icmlauthor{Philipp Honysz}{tu}
    \icmlauthor{Lukas Pfahler}{tu}
    \icmlauthor{Katharina Morik}{tu}
    \end{icmlauthorlist}
    
    \icmlaffiliation{tu}{Artificial Intelligence Group, TU Dortmund University, Germany}
    
    \icmlcorrespondingauthor{Sebastian Buschjäger}{sebastian.buschjaeger@tu-dortmund.de}
    
    \icmlkeywords{Submodular, Streaming, Online, Data Summarization}
    
    \vskip 0.3in
]

\printAffiliationsAndNotice{}

\begin{abstract}
Data summarization has become a valuable tool in understanding even terabytes of data. Due to their compelling theoretical properties, submodular functions have been the focus of summarization algorithms. Submodular function maximization is a well-studied problem with a variety of algorithms available. These algorithms usually offer worst-case guarantees to the expense of higher computation and memory requirements. However, many practical applications do not fall under this mathematical worst-case but are usually much more well-behaved. We propose a new submodular function maximization algorithm called ThreeSieves that ignores the worst-case and thus uses fewer resources. Our algorithm selects the most informative items from a data-stream on the fly and maintains a provable performance in most cases on a fixed memory budget. In an extensive evaluation, we compare our method against 6 state-of-the-art algorithms on 8 different datasets including data with and without concept drift. We show that our algorithm outperforms the current state-of-the-art in the majority of cases and, at the same time, uses fewer resources. 
We make our code publicly available at \url{https://github.com/sbuschjaeger/SubmodularStreamingMaximization}.
\end{abstract}



%

\section{Introduction}

In recent years, submodular optimization has found its way into the toolbox of machine learning and data mining.  Submodular functions reward adding a new element to a smaller set more than adding the same element to a larger set. This makes them ideal for solving data summarization tasks \cite{mirzasoleiman/etal/2016}, active learning \cite{wei/etal/2015}, user recommendation \cite{ashkan/etal/2015}, and many other related tasks. In these tasks, the amount of data is often huge and generated in real-time. Consequently, a line of research studies streaming algorithms for maximizing a submodular function.

In this paper, we consider the problem of maximizing a submodular function over a data stream and focus on the task of data summarization. More formally, we consider the problem of selecting $K$ representative elements from a ground set $V$ into a summary set $S \subseteq V$. To do so, we maximize a non-negative, monotone submodular set function $f \colon 2^V \to \mathbb{R_+}$ which assigns a utility score to each subset:
\begin{equation}
  \label{eq:setmax}
  S^* = \underset{S \subseteq V, |S| = K}{\arg\max} f(S)
\end{equation}
For the empty set, we assume zero utility $f(\emptyset) = 0$. We denote the maximum of $f$ with $OPT = f(S^*)$. A set function can be associated with a marginal gain which represents the increase of $f(S)$ when adding an element $e \in V$ to $S$:
$$
\Delta_f(e | S) = f(S \cup \{e\}) - f(S)
$$
We call $f$ submodular iff for all $A \subseteq B \subseteq V$ and $e \in V \setminus B$ it holds that 
$$
\Delta_f(e | A) \ge \Delta_f(e | B)
$$
The function $f$ is called monotone, iff for all $e \in V$ and for all $S \subseteq V$ it holds that $\Delta_f(e | S) \ge 0$.

The maximization of a submodular set function is NP-hard \cite{Feige/1998} and therefore, a natural approach is to find an approximate solution. 
Table \ref{tab:Alg-Comparison} gives an overview of streaming algorithms which have been proposed for solving Eq. \ref{eq:setmax}. To this date, the best performing online algorithms offer an $\mathcal O(\frac{1}{2} - \varepsilon)$ approximation ratio where $\varepsilon$ also influences the resource consumption. Even moderate choices for $\varepsilon$ quickly result in unmanageable resource consumption. Feldman et al. showed \cite{feldman/etal/2020} that this approximation ratio is the best possible for streaming algorithms and that any algorithm with a better worst-case approximation guarantee essentially stores all the elements of the stream (up to a polynomial factor in $K$). 

We ask whether we can design an algorithm that -- despite the negative result from Feldman et al. -- offers a \emph{better} approximation ratio using \emph{fewer} resources. Existing algorithms are designed for the mathematical \emph{worst-case} and thereby have a worst-case approximation guarantee. We note, that this worse-cast is often a pathological case in their mathematical analysis whereas practical applications are usually much more well-behaved. Thus, we propose to \emph{ignore} these pathological cases and derive an algorithm with a \emph{better} approximation guarantee in \emph{most} cases. 
Our proposed ThreeSieves algorithm estimates the probability of finding a more informative data item on the fly and only adds those items to the summary which are likely to not be `out-valued' in the future. The resulting algorithm offers a \emph{non-deterministic} approximation ratio of $(1-\varepsilon)(1-1/\exp(1)) > \frac{1}{2}-\varepsilon$ in high probability $(1-\alpha)^K$, where $\alpha$ is the desired user certainty. It performs $\mathcal O(1)$ function queries per item and requires $\mathcal O(K)$ memory. Note, that this does not contradict the upper bound of $\frac{1}{2} - \varepsilon$ since our algorithm offers a better approximation quality \emph{in high probability}, but not deterministically for \emph{all} cases.
Our contributions are the following:
\begin{itemize}
  \item The novel ThreeSieves algorithm has an approximation guarantee of $(1-\varepsilon)(1-1/\exp(1))$ in high probability. The fixed memory budget is independent of $\varepsilon$ storing at most $K$ elements and the number of function queries is just one per element.
  \item For the first time we apply submodular function maximization algorithms to data containing concept drift. We show that our ThreeSieves algorithm offers competitive performance in this setting despite its weaker theoretical guarantee. 
  \item We compare our algorithm against 6 state of the art algorithms on 8 datasets with and without concept drift. To the best of our knowledge, this is the first extensive evaluation of state-of-the-art submodular function maximization algorithms in a streaming setting. We show that ThreeSieves outperforms the current state-of-the-art in many cases while being up to $1000$ times faster using a fraction of its memory. 
\end{itemize}
The paper is organized as follows. Section \ref{sec:related-work} surveys related work, whereas in section \ref{sec:three-sieves} we present our main contribution, the ThreeSieves algorithm. Section \ref{sec:experiments-three-sieves} experimentally evaluates ThreeSieves. Section \ref{sec:conclusion} concludes the paper.

\begin{table*}[]
\caption{Algorithms for non-negative, monotone submodular maximization with cardinality constraint $K$. ThreeSieves offers the smallest memory consumption and the smallest number of queries per element in a streaming-setting. \label{tab:Alg-Comparison}}
\centering
\resizebox{\textwidth}{!}{
\begin{tabular}{@{}lcllcc@{}}
\toprule
\textbf{Algorithm} & \textbf{\begin{tabular}[c]{@{}c@{}}Approximation \\ Ratio\end{tabular}}                  & \textbf{Memory}                      & \textbf{\begin{tabular}[c]{@{}l@{}}Queries \\ per Element\end{tabular}} & \textbf{Stream} & \textbf{Ref.}                 \\ \midrule
Greedy             & $1-1/\exp(1)$                                                                                  & $\mathcal O(K)$                      & $\mathcal O(1)$                                                         & \xmark             & \cite{nemhauser/1978}          \\ 
StreamGreedy & $1/2 - \varepsilon$  & $\mathcal O \left( K \right)$ & $\mathcal O(K)$ &\xmark & \cite{gomes/krause/2010}  \\
PreemptionStreaming & $1/4$ & $\mathcal O \left( K \right)$ & $\mathcal O(K)$ &\checkmark & \cite{Buchbinder/etal/2015}  \\
IndependentSetImprovement & $1/4$ & $\mathcal O \left( K \right)$ & $\mathcal O(1)$ &\checkmark & \cite{Chakrabarti/Kale/2014} \\
Sieve-Streaming    & $1/2 - \varepsilon$                                                                      & $\mathcal O(K \log K / \varepsilon)$ & $\mathcal O( \log K / \varepsilon)$                                     & \checkmark         & \cite{Badanidiyuru/etal/2014} \\ 
Sieve-Streaming++  & $1/2 - \varepsilon$                                                                      & $\mathcal O( K / \varepsilon)$       & $\mathcal O( \log K / \varepsilon)$                                     & \checkmark         & \cite{Kazemi/etal/2019}        \\ 
Salsa              & $1/2 - \varepsilon$                                                                      & $\mathcal O(K \log K / \varepsilon)$ & $\mathcal O(\log K / \varepsilon)$                                      & (\checkmark)         & \cite{Norouzi-Fard/etal/2018}  \\ 
QuickStream              & $1/(4c) - \varepsilon$                                                                      & $\mathcal O(c K \log K \log \left ( 1 / \varepsilon \right) )$ & $\mathcal O( \lceil 1 /c \rceil + c )$                                      & \checkmark         & \cite{kuhnle/2021}  \\ 
ThreeSieves        & \begin{tabular}[c]{@{}c@{}}$(1-\varepsilon)(1-1/\exp(1))$ \\ with prob. $(1-\alpha)^K$\end{tabular} & $\mathcal O(K)$                      & $\mathcal O(1)$                                                         & \checkmark         & this paper                    \\ \bottomrule
\end{tabular}
}
\end{table*}

\section{Related Work}
\label{sec:related-work}
 
For a general introduction to submodular function maximization, we refer interested readers to \cite{krause/Golovin/2014} and for a more thorough introduction into the topic of streaming submodular function maximization to \cite{chekuri/etal/2015}. Most relevant to this publication are non-negative, monotone submodular streaming algorithms with cardinality constraints. There exist several algorithms for this problem setting which we survey here. The theoretical properties of each algorithm are summarized in Table \ref{tab:Alg-Comparison}. A detailed formal description including the pseudo-code of each algorithm is given in the appendix. \\ 
While not a streaming algorithm, the Greedy algorithm \cite{nemhauser/1978} forms the basis of many algorithms. It iterates $K$ times over the entire dataset and greedily selects that element with the largest marginal gain $\Delta_f(e | S)$ in each iteration. It offers a $(1-(1/\exp (1))) \approx 63\%$ approximation and stores $K$ elements. StreamGreedy \cite{gomes/krause/2010} is its adaption to streaming data. It replaces an element in the current summary if it improves the current solution by at-least $\nu$. It offers an $\frac{1}{2}-\varepsilon$ approximation with $\mathcal O(K)$ memory, where $\varepsilon$ depends on the submodular function and some user-specified parameters. The optimal approximation factor is only achieved if multiple passes over the data are allowed. Otherwise, the performance of StreamGreedy degrades arbitrarily with $K$ (see Appendix of \cite{Badanidiyuru/etal/2014} for an example). We therefore consider StreamGreedy not to be a proper streaming algorithm. \\
Similar to StreamGreedy, PremptionStreaming \cite{Buchbinder/etal/2015} compares each marginal gain against a threshold $\nu(\mathcal S)$. This time, the threshold dynamically changes depending on the current summary $\mathcal S$ which improves the overall performance. It uses constant memory and offers an approximation guarantee of $1/4$. It was later shown that this algorithm is outperformed by SieveStreaming++ (see below) and was thus not further considered our experiments. 
Chakrabarti and Kale propose in \cite{Chakrabarti/Kale/2014} a streaming algorithm also with approximation guarantee of $1/4$. Their algorithm stores the marginal gain of each element upon its arrival and uses this `weight' to measure the importance of each item. We call this algorithm IndependentSetImprovement. \\
Norouzi-Fard et al. propose in \cite{Norouzi-Fard/etal/2018} a meta-algorithm for submodular function maximization called Salsa which uses different algorithms for maximization as sub-procedures. The authors argue, that there are different types of data-streams and for each stream type, a different thresholding-rule is appropriate. The authors use this intuition to design a $r$-pass algorithm that iterates $r$ times over the entire dataset and adapts the thresholds between each run. They show that their approach is a $(r/(r+1))^r - \varepsilon$ approximation algorithm. For a streaming setting, i.e. $r = 1$, this algorithm recovers the $1/2 - \varepsilon$ approximation bound. However note, that some of the thresholding-rules require additional information about the data-stream such as its length or density. Since this might be unknown in a real-world use-case this algorithm might not be applicable in all scenarios. \\
The first proper streaming algorithm with $1/2 - \varepsilon$ approximation guarantee was proposed by Badanidiyuru et al. in \cite{Badanidiyuru/etal/2014} and is called SieveStreaming. SieveStreaming tries to estimate the potential gain of a data item before observing it. Assuming one knows the maximum function value $OPT$ beforehand and $|S| < K$, an element $e$ is added to the summary $S$ if the following holds:
\begin{equation}
\label{eq:delta}
\Delta_f(e | S) \ge \frac{OPT/2 - f(S)}{K - |S|}
\end{equation}
Since $OPT$ is unknown beforehand one has to estimate it before running the algorithm. Assuming one knows the maximum function value of a singleton set $m = max_{e \in V}f(\{e\})$ beforehand, then the optimal function value for a set with $K$ items can be estimated by submodularity as $m \le OPT \le K \cdot m$.
The authors propose to manage different summaries in parallel, each using one threshold from the set $O = \{(1+\varepsilon)^i \mid i \in \mathbb{Z}, m \le (1+\varepsilon)^i \le K \cdot m\}$, so that for at least one $v \in O$ it holds: $(1-\varepsilon)OPT \le v \le OPT$. In a sense, this approach sieves out elements with marginal gains below the given threshold - hence the authors name their approach SieveStreaming. Note, that this algorithm requires the knowledge of $m = max_{e \in V}f(\{e\})$ before running the algorithm. The authors also present an algorithm to estimate $m$ on the fly which does not alter the theoretical performance of SieveStreaming. 
Recently, Kazemi et al. proposed in \cite{Kazemi/etal/2019} an extension of the SieveStreaming called SieveStreaming++. 
The authors point out, that the currently best performing sieve $S_v = \arg\max_v \{f(S_v)\}$ offers a better lower bound for the function value and they propose to use $[\max_v\{f(S_v)\}, K\cdot m]$ as the interval for sampling thresholds. This results in a more dynamic algorithm, in which sieves are removed once they are outperformed by other sieves and new sieves are introduced to make use of the better estimation of $OPT$. SieveStreaming++ does not improve the approximation guarantee of 
SieveStreaming, but only requires $\mathcal O(K/\varepsilon)$ memory instead of $\mathcal O(K \log K/\varepsilon)$.
Last, Kuhnle proposed the QuickStream algorithm in \cite{kuhnle/2021} which works under the assumption that a single function evaluation is very expensive. QuickStream buffers up to $c$ elements and only evaluates $f$ every $c$ elements. If the function value is increased by the $c$ elements, thy all are added to the solution. Additionally, older examples are removed if there are more than $K$ items in the solution. QuickStream performs well if the costs of a function evaluation $f$ is very costly and if it is ideally independent from the size of the current solution $S$. This is unfortunately not the case in our experiments (see below). Moreover, QuickStream has a guarantee of $1/(4c) - \varepsilon$ that is outperformed by SieveStreaming(++) with similar resource consumption. Thus, we did not consider QuickStream in our experiments.


\section{The Three Sieves Algorithm}
\label{sec:three-sieves}

We recognize, that SieveStreaming and its extensions offer a worst-case guarantee on their performance and indeed they can be consider optimal providing an approximation guarantee of $\frac{1}{2} - \varepsilon$ under polynomial memory constraints \cite{feldman/etal/2020}. However, we also note that this worst case often includes pathological cases, whereas practical applications are usually much more well-behaved. One common practical assumption is, that the data is generated by the same source and thus follows the same distribution (e.g. in a given time frame). In this paper, we want to study these better behaving cases more carefully and present an algorithm which improves the approximation guarantee, while reducing memory and runtime costs in these cases. More formally, we will now assume that the items in the given sample (batch processing) or in the data stream (stream processing) are independent and identically distributed (iid). Note, that we do \textit{not} assume any specific distribution. For batch processing this means, that all items in the data should come from the same (but unknown) distribution and that items should not influence each other. From a data-streams perspective this assumptions means, that the data source will produce items from the same distribution which does not change over time. Hence, we specifically \textit{ignore} concept drift and assume that an appropriate concept drift detection mechanism is in place, so that summaries are e.g. re-selected periodically. We will study streams with drift in more detail in our experimental evaluation. 
We now use this assumption to derive an algorithm with $(1-\varepsilon)(1-1/\exp(1))$ approximation guarantee in high probability:

SieveStreaming and its extension, both, manage $\mathcal O(\log K / \varepsilon)$ sieves in parallel, which quickly becomes unmanageable even for moderate choices of $K$ and $\varepsilon$. We note the following behavior of both algorithms: Many sieves in SieveStreaming have \emph{too small a novelty-threshold} and quickly fill-up with uninteresting events. SieveStreaming++ exploits this insight by removing small thresholds early and thus by focusing on the most promising sieves in the stream. On the other hand, both algorithms manage sieves with \emph{too large a novelty-threshold}, so that they never include any item. Thus, there are only a few thresholds that produce valuable summaries. We exploit this insight with the following approach: Instead of using many sieves with different thresholds we use only a single summary and carefully calibrate the threshold. To do so, we start with a large threshold that rejects most items, and then we gradually reduce this threshold until it accepts some - hopefully the most informative - items. 
The set $O = \{(1+\varepsilon)^i \mid i \in \mathbb{Z}, m \le (1+\varepsilon)^i \le K \cdot m\}$ offers a somewhat crude but sufficient approximation of $OPT$ (c.f. \cite{Badanidiyuru/etal/2014}). We start with the largest threshold in $O$ and decide for each item if we want to add it to the summary or not. If we do not add any of $T$ items (which will be discussed later) to $S$ we may lower the threshold to the next smallest value in $O$ and repeat the process until $S$ is full. 

The key question now becomes: How to choose $T$ appropriately? If $T$ is too small, we will quickly lower the threshold and fill up the summary before any interesting item arrive that would have exceeded the original threshold. If $T$ is too large, we may reject interesting items from the stream. Certainly, we cannot determine with absolute certainty when to lower a threshold without knowing the rest of the data stream or knowing the ground set entirely, but we can do so with high probability. More formally, we aim at estimating the probability $p(e | f, S, v)$ of finding an item $e$ which exceeds the novelty threshold $v$ for a given summary $S$ and function $f$. Once $p$ drops below a user-defined certainty margin $\tau$
$$
p(e | f, S, v) \le \tau
$$
we can safely lower the threshold. This probability must be estimated on the fly.
Most of the time, we reject $e$ so that $S$ and $f(S)$ are unchanged and we keep estimating $p(e | f, S, v)$ based on the negative outcome. If, however, $e$ exceeds the current novelty threshold we add it to $S$ and $f(S)$ changes. In this case, we do not have any estimates for the new summary and must start the estimation of $p(e | f, S, v)$ from scratch. 
Thus, with a growing number of rejected items $p(e | f, S, v)$ tends to become close to $0$ and the key question is how many observations do we need to determine -- with sufficient evidence -- that $p(e | f, S, v)$ will be $0$. 

The computation of confidence intervals for estimated probabilities is a well-known problem in statistics. For example, the confidence interval of binominal distributions can be approximated with normal distributions, Wilson score intervals, or Jeffreys interval. Unfortunately, these methods usually fail for probabilities near $0$  \cite{brown/etal/2001}. However, there exists a more direct way of computing a confidence interval for heavily one-sided binominal distribution with probabilities near zero and iid data \cite{Jovanovic/Levy/1997}. The probability of not adding one item in $T$ trials is: 
$$
\alpha = \left(1-p(e | f, S, v) \right)^T \Leftrightarrow \ln\left(\alpha\right) = T \ln \left(1-p(e | f, S, v) \right)
$$

A first order Taylor Approximation of $\ln(1-p(e | f, S, v) )$ reveals that \\ $\ln \left(1-p(e | f, S, v) \right) \approx -p(e | f, S, v) $
and thus $\ln\left(\alpha\right) \approx T (-p(e | f, S, v) )$ leading to:
\begin{equation}
    \label{eq:RuleOfThree}
  \frac{-\ln\left(\alpha\right)}{T} \approx p(e | f, S, v) \le \tau
\end{equation}

Therefore, the confidence interval of $p(e | f, S, v) $ after observing $T$ events is $\left[0, \frac{-\ln\left(\alpha\right)}{T}\right]$. The $95\%$ confidence interval of $p(e | f, S, v) $ is $\left[0, -\frac{\ln(0.05)}{T}\right]$ which is approximately $[0, 3/T]$ leading to the term ``Rule of Three'' for this estimate \cite{Jovanovic/Levy/1997}. For example, if we did not add any of $T = 1000$ items to the summary, then the probability of adding an item to the summary in the future is below $0.003$ given a $\alpha = 0.95$ confidence interval. We can use the Rule of Three to quantify the certainty that there is a very low probability for finding a novel item in the data stream after observing $T$ items. Note that we can either set $\alpha$, $\tau$ and use Eq. \ref{eq:RuleOfThree} to compute the appropriate $T$ value. Alternatively, we may directly specify $T$ as a user parameter instead of $\alpha$ and $\tau$, thereby effectively removing one hyperparameter. We call our algorithm ThreeSieves and it is depicted in Algorithm \ref{fig:ThreeSieves}. Its theoretical properties are presented in Theorem \ref{th:ThreeSieve}.  
 
\begin{algorithm}
	\begin{algorithmic}[1]
    \STATE $O \gets \{(1+\varepsilon)^i \mid i \in \mathbb{Z}, m \le (1+\varepsilon)^i \le K \cdot m\}$
    \STATE $v \gets \max(O);~O \gets O\setminus\{\max(O)\}$
    \STATE $S \gets \emptyset; ~t \gets 0$
    \FOR{next item $e$}
        \IF{$\Delta_f(e | S) \ge \frac{v/2 - f(S)}{K - |S|}$ and $|S| < K$ }
          \STATE{$S \gets S \cup \{e\};~t \gets 0$}
        \ELSE
          \STATE{$t \gets t + 1$}
          \IF{$t \ge T$}
            \STATE $v \gets \max(O);~O \gets O\setminus\{\max(O)\}; t\gets 0$
          \ENDIF
        \ENDIF
    \ENDFOR	
    \STATE \textbf{return} $S$
	\end{algorithmic}
	\caption{ThreeSieves algorithm.}
	\label{fig:ThreeSieves}
\end{algorithm}

\begin{theorem}
  \label{th:ThreeSieve}
  ThreeSieves has the following properties:
  \begin{itemize}
    \item Given a fixed groundset $V$ or an infinite data-stream in which each item is independent and identically distributed (iid) it outputs a set $S$ such that $|S| \le K$ and with probability $(1 - \alpha)^K$ it holds for a non-negative, monotone submodular function $f$: $f(S) \ge (1- \varepsilon) (1 - 1/\exp(1))OPT$ 
    \item It does $1$ pass over the data (streaming-capable) and stores at most $\mathcal O\left(K\right)$ elements
  \end{itemize}
\end{theorem}

\begin{proof} 
The detailed proof can be found in the appendix. The proof idea is as follows:
The Greedy Algorithm selects that element with the largest marginal gain in each round. Suppose we know the marginal gains $v_i = \Delta(S|e_i)$ for each element $e_i\in S$ selected by Greedy in each round. Then we can simulate the Greedy algorithm by stacking $K$ copies of the dataset consecutively and by comparing the gain $\Delta_f(e|S)$ of each element $e\in V$ against the respective marginal gain $v_i$. Let $O$ be a set of estimated thresholds with $v^*_1,\dots,v^*_K \in O$. Let $v^*_1$ denote the first threshold used by ThreeSieves before any item has been added to $S$. Then by the statistical test of ThreeSieves it holds with probability $1-\alpha$ that
$$
P(v_1 \not= v^*_1) \le \frac{-\ln(\alpha)}{T} \Leftrightarrow P(v_1 = v^*_1) > 1 - \frac{-\ln(\alpha)}{T}
$$
Consequently, it holds with probability $(1-\alpha)^K$ 
$$
P\left(v_1 = v^*_1,\dots,v_K = v^*_K \right) > \left(1-\frac{-\ln(\alpha)}{T}\right)^K
$$
By the construction of $O$ it holds that $(1-\varepsilon) v^*_i \le v_i \le v^*_i$ (c.f. \cite{Badanidiyuru/etal/2014}). 
Let $e_K$ be the element that is selected by ThreeSieves after $K-1$ items have already been selected. Let $S_0 = \emptyset$ and recall that by definition $f(\emptyset) = 0$, then it holds with probability $(1-\alpha)^K$:
\begin{align*}
f(S_K)  &= f(\emptyset) + \sum_{i=1}^K \Delta(e_i|S_{K-1}) \ge \sum_{i=1}^K \left(1-\varepsilon\right) v^*_{i} \\
        &= \left(1-\varepsilon\right) f_{G}(S_K) \ge \left(1-\varepsilon\right) \left(1-1/\exp(1)\right) OPT 
\end{align*}
where $f_{G}$ denotes the solution of the Greedy algorithm. 
\end{proof}

Similar to SieveStreaming, ThreeSieves tries different thresholds until it finds one that fits best for the current summary $S$, the data $V$, and the function $f$. In contrast, however, ThreeSieves is optimized towards minimal memory consumption by maintaining one threshold and one summary at a time. If more memory is available, one may improve the performance of ThreeSieves by running multiple instances of ThreeSieves in parallel on different sets of thresholds. So far, we assumed that we know the maximum singleton value $m = max_{e\in V}f(\{e\})$ beforehand. If this value is unknown before running the algorithm we can estimate it on-the-fly without changing the theoretical guarantees of ThreeSieves. As soon as a new item arrives with a new $m_{new} > m_{old}$ we remove the current summary and use the new upper bound $K \cdot m_{new}$ as the starting threshold. It is easy to see that this does not affect the theoretical performance of ThreeSieves: Assume that a new item arrives with a new maximum single value $m_{new}$. Then, all items in the current summary have a smaller singleton value $m_{old} < m_{new}$. The current summary has been selected based on the assumption that $m_{old}$ was the largest possible value, which was invalidated as soon as $m_{new}$ arrived. Thus, the probability estimate that the first item in the summary would be `out-valued' later in the stream was wrong since we just observed that it is being out-valued. To re-validate the original assumption we delete the current summary entirely and re-start the summary selection. 

\section{Experimental Evaluation}
\label{sec:experiments-three-sieves}

In this section, we experimentally evaluate ThreeSieves and compare it against SieveStreaming(++), IndepdendentSetImprovement, Slasa, and Greedy. As an additional baseline we also consider a random selection of items via Reservoir Sampling \cite{vitter/1985}. We denote this algorithm as Random.
We focus on two different application scenarios: In the first experiment, we have given a batch of data and are tasked to compute a comprehensive summary. In this setting, algorithms are allowed to perform multiple passes over the data to the expense of a longer runtime, e.g. as Greedy does. In the second experiment we shift our focus towards the summary selection on streams with concept drift. Here, each item is only seen once and the algorithms do not have any other information about the data-stream.
All experiments have been run on an Intel Core i7-6700 CPU machine with 8 cores and 32 GB main memory running Ubuntu 16.04. The code for our experiments is available with this submission. It will be publicly available after publication. 

\subsection{Batch experiments}
\label{sec:batch-experiments}

In the batch experiments each algorithm is tasked to select a summary with exactly $K$ elements. Since most algorithms can reject items they may select a summary with less than $K$ elements. To ensure a summary of size $K$, we re-iterate over the entire data-set as often as required until $K$ elements have been selected, but at most $K$ times. We compare the relative maximization performance of all algorithms to the solution of Greedy. For example, a relative performance of $100\%$ means that the algorithm achieved the same performance as Greedy did, whereas a relative performance of $50\%$ means that the algorithm only achieved half the function value of Greedy. We also measure the runtime and memory consumption of each algorithm. The runtime measurements include all re-runs, so that many re-runs over the data-set result in larger runtimes. We evaluate four key questions: First, is ThreeSieves competitive against the other algorithms or will the probabilistic guarantee hurt its practical performance? Second, if ThreeSieves is competitive, how does it related to a Random selection of summaries? Third, how large is the resource consumption of ThreeSieves in comparison? Fourth, how does ThreeSieves behave for different $T$ and different $\varepsilon$?

In total, we evaluate $3895$ hyperparameter configurations on the datasets shown in the top group in Table \ref{tab:datasets}. 
We extract summaries of varying sizes $K \in \{5,10,\dots,100\}$ maximizing the log-determinant $f(S) = \frac{1}{2}\log\det(\mathcal I + a \Sigma_S)$.
Here, $\Sigma_S = [k(e_i,e_j)]_{ij}$ is a kernel matrix containing all similarity pairs of all points in $S$, $a\in\mathbb R_+$ is a scaling parameter and $\mathcal I$ is the identity matrix. In \cite{seeger/2004}, this function is shown to be submodular. Its function value does not depend on $V$, but only on the summary $S$, which makes it an ideal candidate for summarizing data in a streaming setting. In \cite{Buschjaeger/etal/2017}, it is proven that $m = max_{e \in V}f(\{e\}) = 1 + a K$ and that $OPT \le K\log(1+a)$ for kernels with $k(\cdot, \cdot) \le 1$. This property can be enforced for every positive definite kernel with normalization \cite{graf/borer/2001}. 
In our experiments we set $a=1$ and use the RBF kernel $k(e_i, e_j) = \exp\left(-\frac{1}{2l^2}\cdot ||e_i - e_j||^2_2\right)$ with $l = \frac{1}{2\sqrt{d}}$ where $d$ is the dimensionality of the data. We vary $\varepsilon \in \{0.001, 0.005, 0.01, 0.05, 0.1\}$ and $T\in \{50,250,500,1000,2500,5000\}$. 

\begin{table}
  \centering
  \footnotesize
  \caption{Data sets used for the experiments. 
  }
  \begin{tabular}{@{}lrrr@{}}
  \toprule
  Name         & Size & Dim & Reference                 \\ \midrule
  ForestCover & 286,048      & 10     & \cite{DalPozzolo/etal/2015}  \\
  Creditfraud  & 284,807      & 29     & \cite{liu/etal/2008}  \\
  FACT Highlevel         & 200,000      & 16     & \cite{Buschjaeger/etal/2020} \\ 
  FACT Lowlevel         & 200,000      & 256     & \cite{Buschjaeger/etal/2020}  \\ 
  KDDCup99 	& 60,632 &	41 & \cite{campos/etal/2016} \\ \midrule
  stream51 & 150,736     & 2048     & \cite{stream51}  \\
  abc  & 1,186,018      & 300     & \cite{abcnews}  \\
  examiner         & 3,089,781      & 300     & \cite{examiner} \\ \bottomrule
  \end{tabular}
  \label{tab:datasets}
\end{table}

We present two different sets of plots, one for varying $K$ and one for varying $\varepsilon$. Figure \ref{fig:All_over_K} depicts the relative performance, the runtime and the memory consumption over different $K$ for a fixed $\varepsilon = 0.001$. For presentational purposes, we selected $T = 500, 1000, 2500, 5000$ for ThreeSieves. In all experiments, we find that ThreeSieves with $T=5000$ and Salsa generally perform best with a very close performance to Greedy for $K \ge 20$. For smaller summaries with $K< 20$ all algorithms seem to underperform, with Salsa and SieveStreaming performing best. Using $T \le 1000$ for ThreeSieves seems to decrease the performance on some datasets, which can be explained by the probabilistic nature of the algorithm. We also observe a relative performance above $100$ where ThreeSieves performed \textit{better} than Greedy on Creditfraud and Fact Highlevel. Note, that only ThreeSieves showed this behavior, whereas the other algorithms never exceeded Greedy. Expectantly, Random selection shows the weakest performance. SieveStreaming and SieveStreaming++ show identical behavior. Looking at the runtime, please, note the logarithmic scale. Here, we see that ThreeSieves and Random are by far the fastest methods. Using $T = 1000$ offers some performance benefit, but is hardly justified by the decrease in maximization performance, whereas $T = 5000$ is only marginally slower but offers a much better maximization performance. SieveStreaming and SieveStreaming++ have very similar runtime, but are magnitudes slower than Random and ThreeSieves. Last, Salsa is the slowest method. Regarding the memory consumption, please, note again the logarithmic scale. Here, all versions of ThreeSieves use the fewest resources as our algorithm only stores a single summary in all configurations. These curves are identical with Random so that only four instead of 7 curves are to be seen. SieveStreaming and their siblings use roughly two magnitudes more memory since they keep track of multiple sieves in parallel. Salsa uses the most memory, since it keeps track of even more sieves in parallel. 
 
\begin{figure*} 
  \centering
  \includegraphics[width=1.0\textwidth]{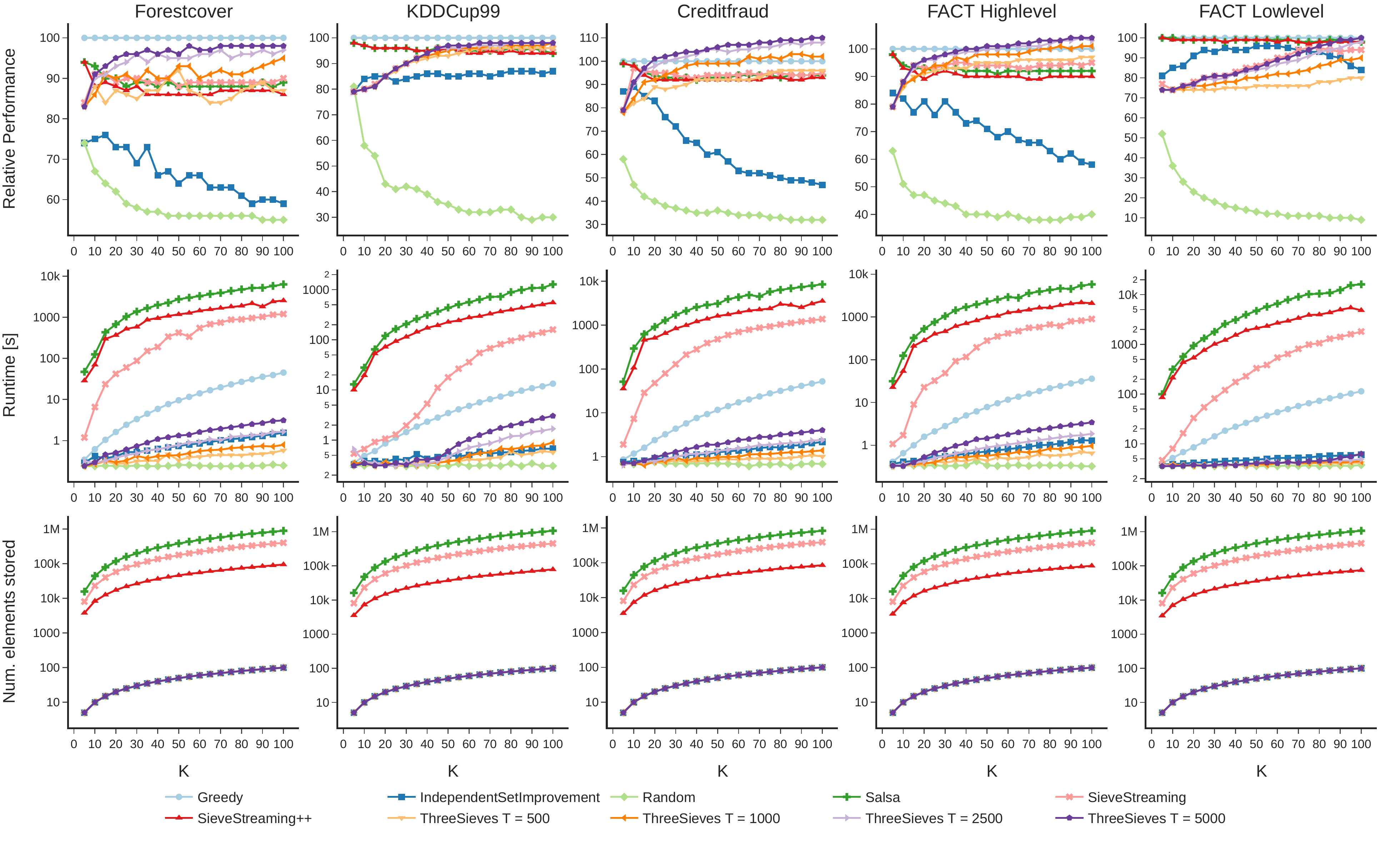}
  \caption[Algorithm comparison over $\varepsilon$.]{Comparison between IndependentSetImprovement, SieveStreaming,SieveStreaming++,Salsa, Random, and ThreeSieves for different $\varepsilon$ values and fixed $K = 50$. The first row shows the relative performance to Greedy (larger is better), the second row shows the total runtime in seconds (logarithmic scale, smaller is better) and the third row shows the maximum memory consumption (logarithmic scale, smaller is better). Each column represents one dataset. \label{fig:All_over_eps} }
  
  \includegraphics[width=1.0\textwidth]{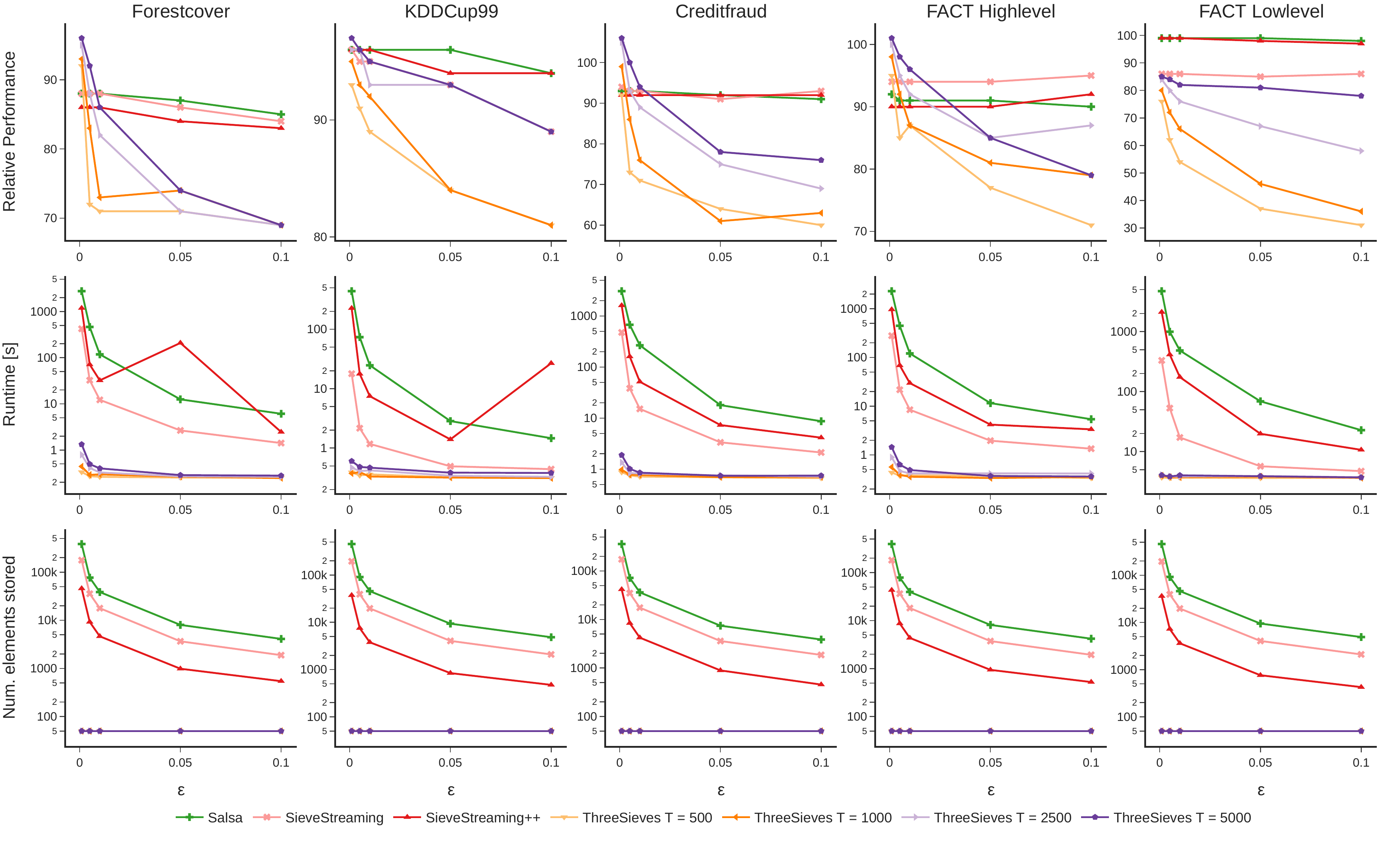}
  \caption[Algorithm comparison over K.]{Comparison between IndependentSetImprovement, SieveStreaming,SieveStreaming++,Salsa, Random, and ThreeSieves for different $K$ values and fixed $\varepsilon = 0.001$. The first row shows the relative performance to Greedy (larger is better), the second row shows the total runtime in seconds (logarithmic scale, smaller is better) and the third row shows the maximum memory consumption (logarithmic scale, smaller is better). Each column represents one dataset.  \label{fig:All_over_K} }
\end{figure*}

Now we look at the behavior of the algorithms for different approximation ratios. Figure \ref{fig:All_over_eps} depicts the relative performance, the runtime, and the memory consumption over different $\varepsilon$ for a fixed $K = 50$. We again selected $T = 500, 1000, 2500, 5000$ for ThreeSieves. Note, that we excluded Random, Greedy and IndependentSetImprovement for these plots since their performance does not vary with $\varepsilon$. Looking at the relative performance we see a slightly different picture than before: For small $\varepsilon \le 0.05$ and larger $T$ we see that ThreeSieves and Salsa again perform best in all cases. For larger  $\varepsilon > 0.05$ the performance of the non-probabilistic algorithms remain relatively stable, but ThreeSieves performance starts to deteriorate. Again we note, that SieveStreaming and SieveStreaming++ show identical behavior. Looking at the runtime and memory consumption we see a similar picture as before: ThreeSieves is by far the fastest method using the fewest resources followed by SieveStreaming(++) and Salsa. Again note, that ThreeSieves requires the same amount of memory in all configurations and hence we find an overlap in plots.


\textbf{We conclude:} In summary, we conclude that ThreeSieves works best for small $\varepsilon$ and large $T$. The probabilistic nature of the algorithm does not decrease its maximization performance but actually helps it in some cases. In contrast to the other algorithms, the resource consumption and overall runtime of ThreeSieves does not suffer from decreasing $\varepsilon$ or increasing $T$. Its overall performance is better or comparable to the other algorithms while being much more memory efficient and overall faster.
 

\subsection{Streaming experiments}
\label{sec:streaming-experiments}

Now we want to compare the algorithms in a true streaming setting including concept drift. Here, we present each item only once and must decide immediately if it should be added to the summary or not. Since Salsa requires additional information about the stream we excluded it for these experiments. We use two real-world data-sets and one artificial data-set depicted in the bottom group in Table \ref{tab:datasets}: The stream51 dataset \cite{stream51} contains image frames from a sequence of videos, where each video shows an object of one of 51 classes. Subsequent frames in the video stream are highly dependent. Over the duration of the stream, new classes are introduced. The dataset is constructed such that online streaming classification methods suffer from  `Catastrophic forgetting' \cite{stream51}. We utilize a pretrained InceptionV3 convolutional neural network that computes 2048-dimensional embeddings of the images.
The abc dataset contain news headlines from the Australian news source `ABC' gathered over 17 years (2003 - 2019) and the examiner dataset contains news headlines from the Australian news source `The Examiner' gathered over 6 years (2010 - 2015). Due to this long time-span we assume that both datasets contain a natural concept drift occurring due to different topics in the news. 
We use pretrained Glove embeddings to extract 300-dimensional, real-valued feature vectors for all headlines and present them in order of appearance starting with the oldest headline. We ask the following questions: First, will the iid assumption of ThreeSieves hurt its practical performance on data-streams with concept drift? Second, how will the other algorithms with a worst-case guarantee perform in this situation?

In total, we evaluate $3780$ hyperparameters on these three datasets. Again, we extract summaries of varying sizes $K \in \{5,10,\dots,100\}$ maximizing the log-determinant with $a=1$. We use the RBF kernel with $l = \frac{1}{\sqrt{d}}$ where $d$ is the dimensionality of the data. We vary $\varepsilon \in \{0.01, 0.1\}$ and $T\in \{500,1000,2500,5000\}$. Again, we report the relative performance of the algorithms compared to Greedy (executed in a batch-fashion). 
Figure \ref{fig:stream_over_k} shows the relative performance of the streaming algorithms for different $K$ with fixed $\varepsilon = 0.1$ (first row) and fixed $\varepsilon = 0.01$ (second row) on the three datasets. On the stream51 dataset we see very chaotic behavior for $\varepsilon = 0.1$. Here, SieveStreaming++ generally seems to be best with a performance around $90 - 95 \%$. ThreeSieves's performance suffers for smaller $T \le 1000$ not exceeding $85\%$. For other configurations with larger $T$ ThreeSieves has a comparable performance to SieveStreaming and IndependentSetImprovement all achieving $85 - 92\%$. For $\varepsilon = 0.01$ the behavior of the algorithms stabilizes. Here we find that ThreeSieves with $T = 5000$ shows a similar and sometimes even better performance compared to SieveStreaming(++) beyond $95\%$. An interesting case occurs for $T = 5000$ and $K = 100$ in which the function value suddenly drops. Here, ThreeSieves rejected more items than were available in the stream and thus returned a summary with less than $K = 100$ items. Somewhere in the middle, we find IndependentSetImprovement reaching $90 \%$ performance in the best case. Expectantly, Random selection is the worst in all cases. In general, we find a similar behavior on the other two datasets. For $\varepsilon = 0.1$, SieveStreaming(++) seem to be the best option followed by ThreeSieves with larger $T$ or IndependentSetImprovement. For larger $K$, IndependentSetImprovement is not as good as ThreeSieves and its performance approaches Random quite rapidly. For $\varepsilon = 0.01$ the same general behavior can be observed. SieveStreaming(++) again holds well under concept drift followed by ThreeSieves with larger $T$ followed by IndependentSetImprovement and Random. We conjecture that ThreeSieves's performance could be further improved for larger $T$ and that there seems to be a dependence between $T$ and the type of concept drift occurring in the data.\\
\noindent \textbf{We conclude:} ThreeSieves holds surprisingly well under concept drift, especially for larger $T$. In many cases its maximization performance is comparable with SieveStreaming(++) while being more resource efficient. For smaller $T$ the performance of ThreeSieves clearly suffers, but remains well over the performance of Random or IndependentSetImprovement. For larger $T$ and smaller $\varepsilon$, ThreeSieves becomes more competitive to SieveStreaming(++) while using fewer resources.. We conclude that ThreeSieves is also applicable for streaming data with concept drift even though its theoretical guarantee does not explicitly hold in this context.  
 
\begin{figure} 
  \centering
  \includegraphics[width=0.5\textwidth, keepaspectratio]{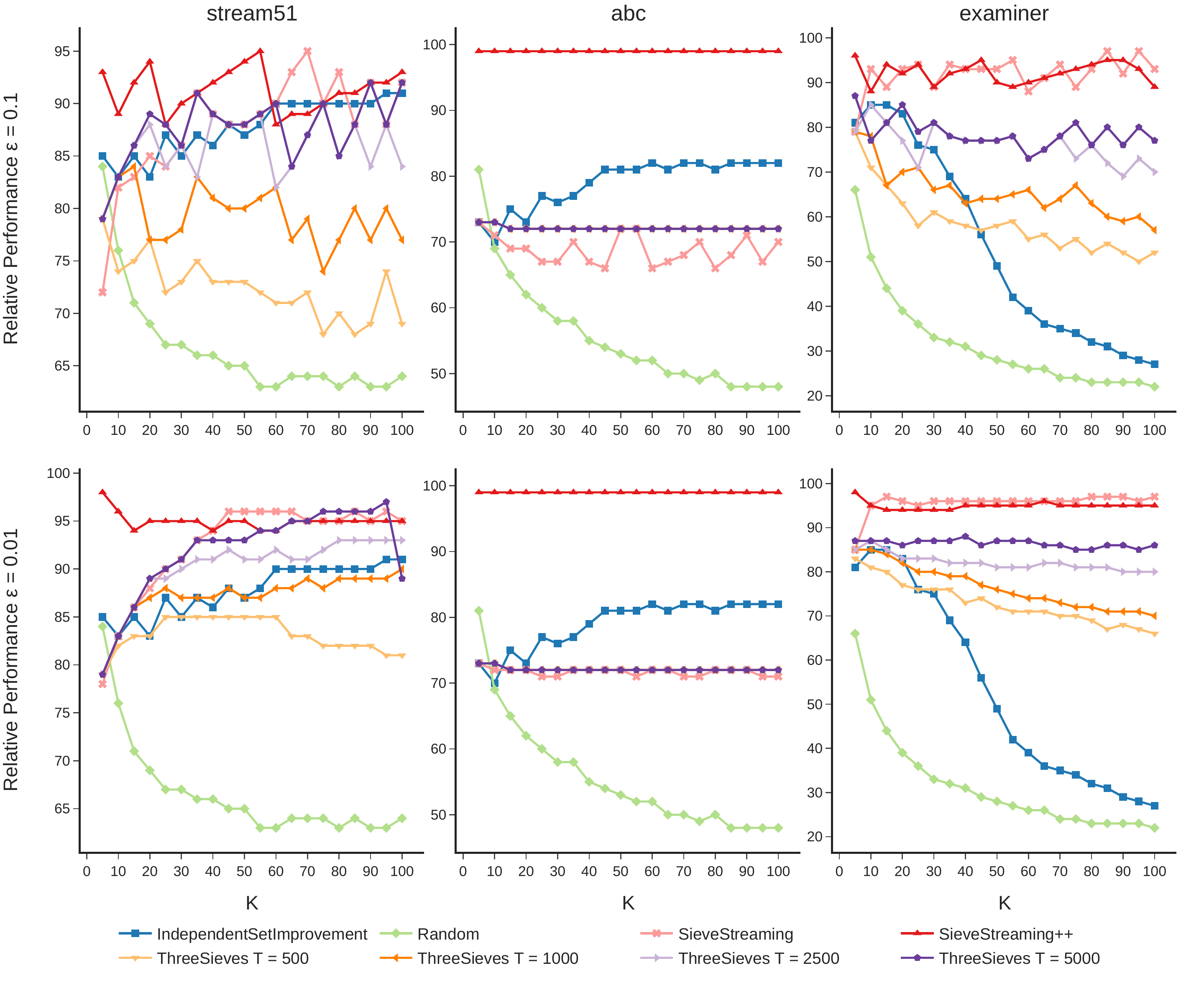}
  \caption[Algorithm comparison over K.]{Comparison between IndependentSetImprovement, SieveStreaming, SieveStreaming++, Random and ThreeSieves for different $K$ values and fixed $\varepsilon = 0.1$ (first row) and $\varepsilon = 0.01$ (second row). Each column represents one dataset.  \label{fig:stream_over_k} }
\end{figure}

 

\section{Conclusion}
\label{sec:conclusion}

Data summarization is an emerging topic for understanding and analyzing large amounts of data. In this paper, we studied the problem of on-the-fly data summarization where one must decide immediately whether to add an item to the summary, or not. While some algorithms for this problem already exist, we recognize that these are optimized towards the worst-case and thereby require more resources. We argue that practical applications are usually much more well-behaved than the commonly analyzed worst-cases. We proposed the ThreeSieves algorithm for non-negative monotone submodular streaming function maximization and showed, that -- under moderate assumptions -- it has an approximation ratio of $(1-\varepsilon)(1-1/\exp(1))$ in high probability. It runs on a fixed memory budget and performs a constant number of function evaluations per item. 
We compared ThreeSieves against 6 state of the art algorithms on 8 different datasets with and without concept drift. For data without concept drift, ThreeSieves outperforms the other algorithms in terms of maximization performance while using two magnitudes less memory and being up to $1000$ times faster. On datasets with concept drift, ThreeSieves outperforms the other algorithms in some cases and offers similar performance in the other cases while being much more resource efficient. This allows for applications, where based on the summary, some action has to be performed. Hence, the novel ThreeSieves algorithm opens up opportunities beyond the human inspection of data summaries, which we want to explore in the future. 






\section*{Acknowledgements}
\label{sec:acknowledgements}

Part of the work on this paper has been supported by Deutsche Forschungsgemeinschaft (DFG) within the Collaborative Research Center SFB 876 "Providing Information by Resource-Constrained Analysis", DFG project number 124020371, SFB project A1, \url{http://sfb876.tu-dortmund.de}. Part of the work on this research has been funded by the Federal Ministry of Education and Research of Germany as part of the competence center for machine learning ML2R (01|18038A), \url{https://www.ml2r.de/}.

We want to thank Sagar Kale for his helpful comments and fruitful discussion.

\bibliographystyle{icml2019}
\bibliography{literature}

\begin{thebibliography}{41}
\providecommand{\natexlab}[1]{#1}
\providecommand{\url}[1]{\texttt{#1}}
\expandafter\ifx\csname urlstyle\endcsname\relax
  \providecommand{\doi}[1]{doi: #1}\else
  \providecommand{\doi}{doi: \begingroup \urlstyle{rm}\Url}\fi

\bibitem[Anderhub et~al.(2011)]{anderhub/etal/2011}
Anderhub, H. et~al.
\newblock Fact—the first cherenkov telescope using a g-apd camera for tev
  gamma-ray astronomy.
\newblock \emph{Nuclear Instruments and Methods in Physics Research Section A:
  Accelerators, Spectrometers, Detectors and Associated Equipment},
  639\penalty0 (1):\penalty0 58--61, 2011.

\bibitem[Anderhub et~al.(2013)]{Anderhub/etal/2013a}
Anderhub, H. et~al.
\newblock {Design and operation of FACT-the first G-APD Cherenkov telescope}.
\newblock \emph{Journal of Instrumentation}, 8\penalty0 (6), 2013.
\newblock ISSN 17480221.

\bibitem[Ashkan et~al.(2015)Ashkan, Kveton, Berkovsky, and
  Wen]{ashkan/etal/2015}
Ashkan, A., Kveton, B., Berkovsky, S., and Wen, Z.
\newblock Optimal greedy diversity for recommendation.
\newblock In \emph{IJCAI}, volume~15, pp.\  1742--1748, 2015.

\bibitem[Badanidiyuru et~al.(2014)Badanidiyuru, Mirzasoleiman, Karbasi, and
  Krause]{Badanidiyuru/etal/2014}
Badanidiyuru, A., Mirzasoleiman, B., Karbasi, A., and Krause, A.
\newblock Streaming submodular maximization: Massive data summarization on the
  fly.
\newblock In \emph{ACM SIGKDD}, 2014.

\bibitem[Biland et~al.(2014)]{Biland/etal/2014a}
Biland, A. et~al.
\newblock {Calibration and performance of the photon sensor response of FACT -
  The first G-APD Cherenkov telescope}.
\newblock \emph{Journal of Instrumentation}, 9\penalty0 (10), 2014.
\newblock ISSN 17480221.

\bibitem[Bockermann et~al.(2015)Bockermann, Br{\"u}gge, Buss, Egorov, Morik,
  Rhode, and Ruhe]{bockermann/etal/2015}
Bockermann, C., Br{\"u}gge, K., Buss, J., Egorov, A., Morik, K., Rhode, W., and
  Ruhe, T.
\newblock Online analysis of high-volume data streams in astroparticle physics.
\newblock In \emph{ECML PKDD}, pp.\  100--115. Springer, 2015.

\bibitem[Brown et~al.(2001)Brown, Cai, and DasGupta]{brown/etal/2001}
Brown, L.~D., Cai, T.~T., and DasGupta, A.
\newblock Interval estimation for a binomial proportion.
\newblock \emph{Statistical science}, pp.\  101--117, 2001.

\bibitem[Buchbinder et~al.(2019)Buchbinder, Feldman, and
  Schwartz]{Buchbinder/etal/2015}
Buchbinder, N., Feldman, M., and Schwartz, R.
\newblock Online submodular maximization with preemption.
\newblock \emph{ACM Trans. Algorithms}, 15\penalty0 (3), June 2019.
\newblock ISSN 1549-6325.
\newblock \doi{10.1145/3309764}.
\newblock URL \url{https://doi.org/10.1145/3309764}.

\bibitem[Buschjäger et~al.(2017)Buschjäger, Morik, and
  Schmidt]{Buschjaeger/etal/2017}
Buschjäger, S., Morik, K., and Schmidt, M.
\newblock Summary extraction on data streams in embedded systems.
\newblock In \emph{ECML Conference Workshop IoT Large Scale Learning from Data
  Streams}, 2017.

\bibitem[Buschjäger et~al.(2020)Buschjäger, Pfahler, Buss, and
  Morik]{Buschjaeger/etal/2020}
Buschjäger, S., Pfahler, L., Buss, J., and Morik, K.
\newblock On-site gamma-hadron separation with deep learning on fpgas
  (accepted).
\newblock In \emph{Machine Learning and Knowledge Discovery in Databases (ECML
  PKDD)}. Springer International Publishing, 2020.

\bibitem[Campos et~al.(2016)Campos, Zimek, Sander, Campello, Micenkov{\'a},
  Schubert, Assent, and Houle]{campos/etal/2016}
Campos, G.~O., Zimek, A., Sander, J., Campello, R.~J., Micenkov{\'a}, B.,
  Schubert, E., Assent, I., and Houle, M.~E.
\newblock On the evaluation of unsupervised outlier detection: measures,
  datasets, and an empirical study.
\newblock \emph{Data Mining and Knowledge Discovery}, 30\penalty0 (4):\penalty0
  891--927, 2016.

\bibitem[Chakrabarti \& Kale(2014)Chakrabarti and Kale]{Chakrabarti/Kale/2014}
Chakrabarti, A. and Kale, S.
\newblock Submodular maximization meets streaming: Matchings, matroids, and
  more.
\newblock In Lee, J. and Vygen, J. (eds.), \emph{Integer Programming and
  Combinatorial Optimization}, pp.\  210--221, Cham, 2014. Springer
  International Publishing.
\newblock ISBN 978-3-319-07557-0.

\bibitem[Chekuri et~al.(2015)Chekuri, Gupta, and Quanrud]{chekuri/etal/2015}
Chekuri, C., Gupta, S., and Quanrud, K.
\newblock Streaming algorithms for submodular function maximization.
\newblock In \emph{International Colloquium on Automata, Languages, and
  Programming}, pp.\  318--330. Springer, 2015.

\bibitem[Dal~Pozzolo et~al.(2015)Dal~Pozzolo, Caelen, Johnson, and
  Bontempi]{DalPozzolo/etal/2015}
Dal~Pozzolo, A., Caelen, O., Johnson, R.~A., and Bontempi, G.
\newblock Calibrating probability with undersampling for unbalanced
  classification.
\newblock In \emph{2015 IEEE Symposium Series on Computational Intelligence},
  pp.\  159--166. IEEE, 2015.
\newblock URL \url{https://www.kaggle.com/mlg-ulb/creditcardfraud}.

\bibitem[Feige(1998)]{Feige/1998}
Feige, U.
\newblock A threshold of ln n for approximating set cover.
\newblock \emph{J. ACM}, 45\penalty0 (4):\penalty0 634--652, July 1998.
\newblock ISSN 0004-5411.
\newblock \doi{10.1145/285055.285059}.
\newblock URL \url{http://doi.acm.org/10.1145/285055.285059}.

\bibitem[Feige et~al.(2011)Feige, Mirrokni, and Vondr{\'a}k]{feige/etal/2011}
Feige, U., Mirrokni, V.~S., and Vondr{\'a}k, J.
\newblock Maximizing non-monotone submodular functions.
\newblock \emph{SIAM Journal on Computing}, 40\penalty0 (4):\penalty0
  1133--1153, 2011.

\bibitem[Feldman et~al.(2020)Feldman, Norouzi-Fard, Svensson, and
  Zenklusen]{feldman/etal/2020}
Feldman, M., Norouzi-Fard, A., Svensson, O., and Zenklusen, R.
\newblock The one-way communication complexity of submodular maximization with
  applications to streaming and robustness.
\newblock In \emph{Proceedings of the 52nd Annual ACM SIGACT Symposium on
  Theory of Computing}, pp.\  1363--1374, 2020.

\bibitem[Goemans et~al.(2009)Goemans, Harvey, Iwata, and
  Mirrokni]{goemans/etal/2009}
Goemans, M.~X., Harvey, N.~J., Iwata, S., and Mirrokni, V.
\newblock Approximating submodular functions everywhere.
\newblock In \emph{Proceedings of the twentieth annual ACM-SIAM symposium on
  Discrete algorithms}, pp.\  535--544, 2009.

\bibitem[Gomes \& Krause(2010)Gomes and Krause]{gomes/krause/2010}
Gomes, R. and Krause, A.
\newblock Budgeted nonparametric learning from data streams.
\newblock In \emph{ICML}, volume~1, pp.\ ~3, 2010.

\bibitem[Graf \& Borer(2001)Graf and Borer]{graf/borer/2001}
Graf, A.~B. and Borer, S.
\newblock Normalization in support vector machines.
\newblock In \emph{DAGM Symposium of Pattern Recognition}, 2001.

\bibitem[Iyer et~al.(2013)Iyer, Jegelka, and Bilmes]{iyer/etal/2013}
Iyer, R.~K., Jegelka, S., and Bilmes, J.~A.
\newblock Curvature and optimal algorithms for learning and minimizing
  submodular functions.
\newblock In \emph{NIPS}, 2013.

\bibitem[Jovanovic \& Levy(1997)Jovanovic and Levy]{Jovanovic/Levy/1997}
Jovanovic, B.~D. and Levy, P.~S.
\newblock A look at the rule of three.
\newblock \emph{The American Statistician}, 51\penalty0 (2):\penalty0 137--139,
  1997.
\newblock \doi{10.1080/00031305.1997.10473947}.
\newblock URL
  \url{https://www.tandfonline.com/doi/abs/10.1080/00031305.1997.10473947}.

\bibitem[Kazemi et~al.(2019)Kazemi, Mitrovic, Zadimoghaddam, Lattanzi, and
  Karbasi]{Kazemi/etal/2019}
Kazemi, E., Mitrovic, M., Zadimoghaddam, M., Lattanzi, S., and Karbasi, A.
\newblock Submodular streaming in all its glory: Tight approximation, minimum
  memory and low adaptive complexity.
\newblock In \emph{ICML}, pp.\  3311--3320, 2019.
\newblock URL \url{http://proceedings.mlr.press/v97/kazemi19a.html}.

\bibitem[Krause \& Golovin(2014)Krause and Golovin]{krause/Golovin/2014}
Krause, A. and Golovin, D.
\newblock Submodular function maximization., 2014.
\newblock URL
  \url{http://www.cs.cmu.edu/afs/.cs.cmu.edu/Web/People/dgolovin/papers/submodular_survey12.pdf}.

\bibitem[Kuhnle(2021)]{kuhnle/2021}
Kuhnle, A.
\newblock Quick streaming algorithms for maximization of monotone submodular
  functions in linear time, 2021.

\bibitem[Kulesza et~al.(2012)Kulesza, Taskar, et~al.]{Kulesza/Taskar/2012}
Kulesza, A., Taskar, B., et~al.
\newblock Determinantal point processes for machine learning.
\newblock \emph{Foundations and Trends{\textregistered} in Machine Learning},
  5\penalty0 (2--3):\penalty0 123--286, 2012.

\bibitem[Kulkarni(2017{\natexlab{a}})]{abcnews}
Kulkarni, R.
\newblock A million news headlines -- news headlines published over a period of
  17 years, 2017{\natexlab{a}}.
\newblock URL \url{https://www.kaggle.com/therohk/million-headlines}.

\bibitem[Kulkarni(2017{\natexlab{b}})]{examiner}
Kulkarni, R.
\newblock The examiner - spam clickbait catalog -- 6 years of crowd sourced
  journalism, 2017{\natexlab{b}}.
\newblock URL \url{https://www.kaggle.com/therohk/examine-the-examiner}.

\bibitem[Lawrence et~al.(2003)Lawrence, Seeger, Herbrich,
  et~al.]{lawrence/etal/2003}
Lawrence, N., Seeger, M., Herbrich, R., et~al.
\newblock Fast sparse gaussian process methods: The informative vector machine.
\newblock \emph{NIPS}, 2003.

\bibitem[Liu et~al.(2008)Liu, Ting, and Zhou]{liu/etal/2008}
Liu, F.~T., Ting, K.~M., and Zhou, Z.-H.
\newblock Isolation forest.
\newblock In \emph{2008 Eighth IEEE International Conference on Data Mining},
  pp.\  413--422. IEEE, 2008.
\newblock URL
  \url{http://odds.cs.stonybrook.edu/forestcovercovertype-dataset/}.

\bibitem[Mirzasoleiman et~al.(2016)Mirzasoleiman, Badanidiyuru, and
  Karbasi]{mirzasoleiman/etal/2016}
Mirzasoleiman, B., Badanidiyuru, A., and Karbasi, A.
\newblock Fast constrained submodular maximization: Personalized data
  summarization.
\newblock In \emph{ICML}, pp.\  1358--1367, 2016.

\bibitem[Mueller et~al.(2017)]{Mueller/etal/2017a}
Mueller, S. et~al.
\newblock {Single photon extraction for FACT's SiPMs allows for novel IACT
  event representation}.
\newblock In \emph{Proceedings of Science}, 2017.

\bibitem[Nemhauser et~al.(1978)]{nemhauser/1978}
Nemhauser, G. et~al.
\newblock An analysis of approximations for maximizing submodular set
  functions—i.
\newblock \emph{Mathematical Programming}, 1978.

\bibitem[Norouzi-Fard et~al.(2018{\natexlab{a}})Norouzi-Fard, Tarnawski,
  Mitrovic, Zandieh, Mousavifar, and Svensson]{Norouzi-Fard/etal/2018}
Norouzi-Fard, A., Tarnawski, J., Mitrovic, S., Zandieh, A., Mousavifar, A., and
  Svensson, O.
\newblock Beyond 1/2-approximation for submodular maximization on massive data
  streams.
\newblock In \emph{ICML}, pp.\  3829--3838, 10--15 Jul 2018{\natexlab{a}}.
\newblock URL \url{http://proceedings.mlr.press/v80/norouzi-fard18a.html}.

\bibitem[Norouzi-Fard et~al.(2018{\natexlab{b}})Norouzi-Fard, Tarnawski,
  Mitrović, Zandieh, Mousavifar, and Svensson]{Norouzi-Fard/etal/2018b}
Norouzi-Fard, A., Tarnawski, J., Mitrović, S., Zandieh, A., Mousavifar, A.,
  and Svensson, O.
\newblock Beyond $1/2$-approximation for submodular maximization on massive
  data streams, 2018{\natexlab{b}}.

\bibitem[Paszke et~al.(2019)]{Paszke/etal/2019}
Paszke, A. et~al.
\newblock Pytorch: An imperative style, high-performance deep learning library.
\newblock In \emph{NeurIP 32}, pp.\  8024--8035. 2019.

\bibitem[Roady et~al.(2020)Roady, Hayes, Vaidya, and Kanan]{stream51}
Roady, R., Hayes, T.~L., Vaidya, H., and Kanan, C.
\newblock Stream-51: Streaming classification and novelty detection from
  videos.
\newblock In \emph{Proceedings of the IEEE/CVF Conference on Computer Vision
  and Pattern Recognition (CVPR) Workshops}, June 2020.

\bibitem[Seeger(2004)]{seeger/2004}
Seeger, M.
\newblock Greedy forward selection in the informative vector machine.
\newblock Technical report, Technical report, University of California at
  Berkeley, 2004.

\bibitem[Vitter(1985)]{vitter/1985}
Vitter, J.~S.
\newblock Random sampling with a reservoir.
\newblock \emph{ACM Transactions on Mathematical Software (TOMS)}, 11\penalty0
  (1):\penalty0 37--57, 1985.

\bibitem[Wei et~al.(2015)Wei, Iyer, and Bilmes]{wei/etal/2015}
Wei, K., Iyer, R., and Bilmes, J.
\newblock Submodularity in data subset selection and active learning.
\newblock In \emph{International Conference on Machine Learning}, pp.\
  1954--1963, 2015.

\bibitem[Williams \& Rasmussen(2006)Williams and
  Rasmussen]{Williams/Rasmussen/2006}
Williams, C.~K. and Rasmussen, C.~E.
\newblock \emph{Gaussian processes for machine learning}, volume~2.
\newblock MIT press Cambridge, MA, 2006.

\end{thebibliography}

\onecolumn
\section*{APPENDIX}
This appendix accompanies the paper `Very Fast Streaming Submodular Function Maximization'. It provides more background information which is not given in the paper due to space reasons. Sections are meant as a drop-in replacement for the original sections in the paper containing most of the original formulations as well as extra explanations and code descriptions. Note that we added a new section on concrete examples of submodular functions for streaming algorithms, which partly appeared in the beginning of the experimental evaluation section of the original paper. Also we highlight a real-world use-case of data summarization applied to high energy astrophysics data. We extract data summaries over real-world data from a cherenkov telescope and present the resulting summaries to a domain expert who gives meaningful annotations to the selected items. 

\section{Introduction}
No changes. 

\section{Related Work}
\label{sec:related-work-appendix}

\begin{table}[h]
\centering
\resizebox{\textwidth}{!}{
\begin{tabular}{@{}lcllcc@{}}
\toprule
\textbf{Algorithm} & \textbf{\begin{tabular}[c]{@{}c@{}}Approximation \\ Ratio\end{tabular}}                  & \textbf{Memory}                      & \textbf{\begin{tabular}[c]{@{}l@{}}Queries \\ per Element\end{tabular}} & \textbf{Stream} & \textbf{Ref.}                 \\ \midrule
Greedy             & $1-1/\exp(1)$                                                                                  & $\mathcal O(K)$                      & $\mathcal O(1)$                                                         & \xmark             & \cite{nemhauser/1978}          \\ 
StreamGreedy & $1/2 - \varepsilon$  & $\mathcal O \left( K \right)$ & $\mathcal O(K)$ &\xmark & \cite{gomes/krause/2010}  \\
PreemptionStreaming & $1/4$ & $\mathcal O \left( K \right)$ & $\mathcal O(K)$ &\checkmark & \cite{Buchbinder/etal/2015}  \\
IndependentSetImprovement & $1/4$ & $\mathcal O \left( K \right)$ & $\mathcal O(1)$ &\checkmark & \cite{Chakrabarti/Kale/2014} \\
Sieve-Streaming    & $1/2 - \varepsilon$                                                                      & $\mathcal O(K \log K / \varepsilon)$ & $\mathcal O( \log K / \varepsilon)$                                     & \checkmark         & \cite{Badanidiyuru/etal/2014} \\ 
Sieve-Streaming++  & $1/2 - \varepsilon$                                                                      & $\mathcal O( K / \varepsilon)$       & $\mathcal O( \log K / \varepsilon)$                                     & \checkmark         & \cite{Kazemi/etal/2019}        \\ 
Salsa              & $1/2 - \varepsilon$                                                                      & $\mathcal O(K \log K / \varepsilon)$ & $\mathcal O(\log K / \varepsilon)$                                      & (\checkmark)         & \cite{Norouzi-Fard/etal/2018}  \\ 
QuickStream              & $1/(4c) - \varepsilon$                                                                      & $\mathcal O(c K \log K \log \left ( 1 / \varepsilon \right) )$ & $\mathcal O( \lceil 1 /c \rceil + c )$                                      & \checkmark         & \cite{kuhnle/2021}  \\ 
ThreeSieves        & \begin{tabular}[c]{@{}c@{}}$(1-\varepsilon)(1-1/\exp(1))$ \\ with prob. $(1-\alpha)^K$\end{tabular} & $\mathcal O(K)$                      & $\mathcal O(1)$                                                         & \checkmark         & this paper                    \\ \bottomrule
\end{tabular}
}
\caption{Algorithms for non-negative, monotone submodular maximization with cardinality constraint $K$. ThreeSieves offers the smallest memory consumption and the smallest number of queries per element in a streaming-setting. \label{tab:Alg-Comparison-appendix}}
\end{table}

For a general introduction to submodular function maximization, we refer interested readers to Krause and Golovin's survey \cite{krause/Golovin/2014} and for a more thorough introduction into the topic of streaming submodular function maximization to \cite{chekuri/etal/2015}. Most relevant to this publication are non-negative, monotone submodular streaming algorithms with cardinality constraints. To the best of our knowledge, there exist seven different algorithms which we quickly survey here. For convenience, we also include the Greedy algorithm in this overview.

\subsection{Greedy}
As the name suggests, the Greedy algorithm works in a simple greedy fashion: It starts with an empty summary $S$. Then it rates each element in $V$ according to its marginal gain $\Delta_f(e | S)$ and picks that element $e$ with the largest gain. This process is repeated $K$ times until the summary is full. A formal algorithm description is given in Figure \ref{fig:Greedy}. Greedy does not have  streaming capabilities.

\begin{algorithm}
\begin{algorithmic}[1]
  \STATE $S \gets \emptyset$
  \FOR{$1,\dots,K$}
    \STATE $e=\arg\max\{\Delta_f(e | S) | e\in V\}$
    \STATE $S \gets S\cup \{e\}$
  \ENDFOR
  \STATE \textbf{return} $S$
\end{algorithmic}
\caption{Greedy algorithm.}
\label{fig:Greedy}
\end{algorithm}

\begin{theorem}
  \label{th:Greedy-appendix}
  Greedy has the following properties \cite{nemhauser/1978}:
  \begin{itemize}
    \item It outputs a set $S$ such that $|S| = K$ and $f(S) \ge (1-(1/\exp(1)))OPT$
    \item It does $K$ passes over the data and stores at most $\mathcal O(K)$ elements
  \end{itemize}
\end{theorem}

\subsection{Random} 
When runtime is a major concern a rather extreme idea is to randomly sample a solution instead of computing it.
Feige et al. study the behavior of random algorithms for monontone submodular function maximization more closely in \cite{feige/etal/2011}. They show, that for any nonnegative, unconstrained submodular function maximization problem a uniformly random set is a $1/4$ approximation in expectation. While such a strong guarantee is not known for cardinality constraints, it still can be interesting to consider the empirical performance of random sets as a simple baseline. A random sample can be obtained in the following fashion: We start by unconditionally accepting the first $K$ elements and use Reservoir Sampling to get a uniform sample over the data-stream \cite{vitter/1985}. Reservoir Sampling randomly replaces an item in the current summary with decreasing probability by sampling the replacement index $j$ form a uniform distribution with growing support. A formal description of the algorithm is given in Figure \ref{fig:Random-appendix}.

\begin{algorithm}
	\begin{algorithmic}[1]
	    \STATE $S \gets \emptyset; i \gets 0$
		\FOR{next item $e$}
    		\IF{$|S| \le K$}
    		    \STATE{$S \gets S \cup \{e\}$}
    		\ELSE
    		    \STATE{$j \gets \text{uniform\_int}(1,i)$}
    		    \IF{$j \le K$}
    		        \STATE{$e_j \gets S[j-1]$}
    		        \STATE{$S \gets S \setminus \{ e_j \} \cup \{e\}$}
    		    \ENDIF
    		\ENDIF
    		\STATE{$i \gets i + 1$}
		\ENDFOR
		\STATE \textbf{return} $S$
	\end{algorithmic}
	\caption{Random algorithm.}
	\label{fig:Random-appendix}
\end{algorithm}

\begin{theorem}
  \label{th:Random-appendix}
  Random has the following properties:
  \begin{itemize}
    \item It outputs a set $S$ such that $|S| \le K$ but offers no guarantee on the quality of the $f(S)$
    \item It does $1$ pass over the data and stores at most $\mathcal O\left(K \right)$ elements
  \end{itemize}
\end{theorem}

\subsection{IndependentSetImprovement} 

Chakrabarti and Kale propose in \cite{Chakrabarti/Kale/2014} a simple streaming algorithm for submodular function maximization in the context of stream graph matching. Here,  nodes and edges of the graph are revealed one by one as a data stream. Their algorithm focuses on matroid constraints opposed to cardinality constraints, but is also applicable for cardinality constraints as discussed in the appendix of \cite{Chakrabarti/Kale/2014}. The main idea of this algorithm is to store the marginal gain of each element upon its arrival and use this `weight' to measure the importance of each item. To do so, it first accepts all $K$ elements unconditionally. Then, once a new item arrives, it replaces that element in the summary with the smallest weight if the current items weight is at-least twice as large. Note, that this algorithm only stores the weight \textit{at the moment} of insertion, but never updates any of the weights if items in the summary are replaced. Hence, we call this algorithm IndependentSetImprovement. Figure \ref{fig:IndependentSetImprovement-appendix} depicts the pseudocode for this algorithm and Theorem \ref{th:IndependentSetImprovement-appendix} shows its theoretical properties. Note, that the set $W$ (line 9) can be implemented efficiently with a priority queue. Although the approximation ratio of $1/4$ seems weak - note that this approximation ratio has been proven in the much more general framework of matroid constraints. 

\begin{algorithm}
	\begin{algorithmic}[1]
	    \STATE $S \gets \emptyset$
	    \STATE $W \gets \emptyset$
		\FOR{next item $e$}
	        \STATE $w_e \gets \Delta_f(e|S)$
    		\IF{$|S| \le K$}
    		    \STATE{$S \gets S \cup \{e\}$}
    		    \STATE{$W \gets W \cup \{w_e\}$}
    		\ELSE
    		    \STATE{$w_m \gets \min W$}
    		    \IF{$w_e > 2w_m $}
    		        \STATE{$S \gets S \setminus \{ m \} \cup \{e\}$}
        		    \STATE{$W \gets W \setminus \{ w_m \} \cup \{w_e\}$}
        		\ENDIF
    		\ENDIF
		\ENDFOR
		\STATE \textbf{return} $S$
	\end{algorithmic}
	\caption{IndependentSetImprovement algorithm.}
	\label{fig:IndependentSetImprovement-appendix}
\end{algorithm}
 
\begin{theorem}
  \label{th:IndependentSetImprovement-appendix}
  IndependentSetImprovement has the following properties \cite{Chakrabarti/Kale/2014}:
  \begin{itemize}
    \item It outputs a set $S$ such that $|S| \le K$ and $f(S) \ge 1/4 OPT$
    \item It does 1 pass over the data and stores at most $\mathcal O\left(K \right)$ elements
  \end{itemize}
\end{theorem}

\subsection{StreamGreedy} 
Gomes and Krause present in \cite{gomes/krause/2010} one of the first works for submodular function maximization called StreamGreedy. Their approach offers an $\frac{1}{2}-\varepsilon$ approximation in $\mathcal O(K)$ memory, where $\varepsilon$ depends on the submodular function and some user-specified parameters. StreamGreedy unconditionally accepts the first $K$ elements into the summary $S$. After that, it swaps newly arrived elements $e$ with any element in $S$ if this improves the current solution by a constant threshold $\nu$. A formal description of the algorithm is given in Figure \ref{fig:StreamGreedy-appendix}. Note, that the original formulation of SimpleGreedy in \cite{gomes/krause/2010} stops the summary selection if there was no improvement for $NI$ iterations. This is not applicable for a true streaming setting in which elements arrive one-by-one and was thus omitted from the pseudocode. Also note, that the optimal approximation factor is only achieved if multiple passes over the data are allowed. 
Otherwise, the performance of StreamGreedy degrades  arbitrarily with $K$ (see Appendix of \cite{Badanidiyuru/etal/2014} for an example).
Since StreamGreedy is not a proper streaming algorithm we did not consider it for our experiments in this paper. 

\begin{algorithm}
	\begin{algorithmic}[1]
	    \STATE $S \gets \emptyset$
		\FOR{next item $e$}
    		\IF{$|S| \le K$}
    		    \STATE{$S \gets S \cup \{e\}$}
    		\ELSE
        		\STATE{$u \gets \arg\max\{f(S \setminus \{ v\} \cup \{e\})|v \in S\}$}
        		\IF{$f(S \setminus \{ u \} \cup \{e\}) - f(S) \ge  \nu$}
        		    \STATE{$S \gets S \setminus \{ u \} \cup \{e\}$}
        		\ENDIF
    		\ENDIF
		\ENDFOR
		\STATE \textbf{return} $S$
	\end{algorithmic}
	\caption{StreamGreedy algorithm.}
	\label{fig:StreamGreedy-appendix}
\end{algorithm}

\begin{theorem}
  \label{th:StreamGreedy-appendix}
  StreamGreedy has the following properties \cite{gomes/krause/2010}:
  \begin{itemize}
    \item It outputs a set $S$ such that $|S| \le K$ and $f(S) \ge \left( \frac{1}{2} - \varepsilon\right) OPT$
    \item It does multiple passes over the data and stores at most $\mathcal O\left(K \right)$ elements
  \end{itemize}
\end{theorem}


\subsection{PremptionStreaming} 
Buchbinder et al. proposed in \cite{Buchbinder/etal/2015} an algorithm which we call PremptionStreaming. It uses constant memory, but only offers an approximation guarantee of $1/4$. It works similar to StreamGreedy but instead of using a fixed threshold $\nu$ it uses a more dynamic threshold ratio $c/f(S)$ which depends on the current function value.
PremptionStreaming unconditionally accepts the first $K$ elements into the summary $S$. After that, it swaps newly arrived elements $e$ with any element in $S$ if this improves the current solution by a given threshold ratio. Formally, Buchbinder et al. introduce the user parameter $c$ and check the current solution against $c/f(S)$. The authors show that the quality of the obtained solution is at-least $\frac{c}{(c+1)^2}$ leading to the guarantee of $1/4$ for $c = 1$. A formal description of the algorithm is given in Figure \ref{fig:PremptionStreaming-appendix}.
It was later shown that this algorithm is outperformed by SieveStreaming++ \cite{Kazemi/etal/2019} and was thus not further considered in this publication. 

\begin{algorithm}
	\begin{algorithmic}[1]
	    \STATE $S \gets \emptyset$
		\FOR{next item $e$}
    		\IF{$|S| \le K$}
    		    \STATE{$S \gets S \cup \{e\}$}
    		\ELSE
        		\STATE{$u \gets \arg\max\{f(S \setminus \{ v\} \cup \{e\})|v \in S\}$}
        		\IF{$f(S \setminus \{ u \} \cup \{e\}) - f(S) \ge \frac{f(S)}{K}$}
        		    \STATE{$S \gets S \setminus \{ u \} \cup \{e\}$}
        		\ENDIF
    		\ENDIF
		\ENDFOR
		\STATE \textbf{return} $S$
	\end{algorithmic}
	\caption{PremptionStreaming algorithm.}
	\label{fig:PremptionStreaming-appendix}
\end{algorithm}

\begin{theorem}
  \label{th:PremptionStreaming-appendix}
  PremptionStreaming has the following properties \cite{Buchbinder/etal/2015}:
  \begin{itemize}
    \item It outputs a set $S$ such that $|S| \le K$ and $f(S) \ge \frac{1}{4}OPT$
    \item It does $1$ pass over the data and stores at most $\mathcal O\left(K \right)$ elements
  \end{itemize}
\end{theorem}

\subsection{SieveStreaming}

Badanidiyuru et al. propose in  \cite{Badanidiyuru/etal/2014} the SieveStreaming algorithm which tries to estimate the potential gain of a data item before observing it.
Assuming one knows the maximum function value $OPT$ beforehand and $|S| < K$, an element $e$ is added to the summary $S$ if the following holds:
\begin{equation}
\label{eq:delta-appendix}
\Delta_f(e | S) \ge \frac{OPT/2 - f(S)}{K - |S|}
\end{equation}

Since $OPT$ is unknown beforehand one has to estimate it before running the algorithm. Assuming one knows the maximum function value of a singleton set $m = max_{e \in V}f(\{e\})$ beforehand, then the optimal function value for a set with $K$ items can be estimated by submodularity as
\begin{equation}
\label{eq:OPT-range-appendix}
m \le OPT \le K \cdot m
\end{equation}

We can use this range to sample different threshold values from the interval $[m, Km]$ such that one of these thresholds will be close to $OPT$. More formally the authors propose to manage different summaries in parallel, each using one threshold from the set $O = \{(1+\varepsilon)^i \mid i \in \mathbb{Z}, m \le (1+\varepsilon)^i \le K \cdot m\}$, so that for at least one
$v \in O$ it holds: $(1-\varepsilon)OPT \le v \le OPT$. In a sense, this approach sieves out elements with marginal gains below the given threshold - hence the authors name their approach SieveStreaming. The SieveStreaming algorithm is depicted in Algorithm \ref{fig:SieveStreaming-appendix} and its theoretical properties are summarized in Table \ref{tab:Alg-Comparison-appendix}. Please note, that this algorithm requires the knowledge of $m = max_{e \in V}f(\{e\})$ before running the algorithm. Badanidiyuru et al. also present an algorithm to estimate $m$ on the fly which does not alter the theoretical performance of SieveStreaming. 

\begin{algorithm}
	\begin{algorithmic}[1]
		\STATE $O \gets \{(1+\varepsilon)^i \mid i \in \mathbb{Z}, m \le (1+\varepsilon)^i \le K \cdot m\}$
		\FOR{$v \in O$}
		\STATE $S_v \gets \emptyset$
		\ENDFOR
		\FOR{next item $e$}
		\FOR{$v \in O$}
		\IF{$\Delta_f(e | S_v) \ge \frac{v/2 - f(S_v)}{K - |S_v|}$ and $|S_v| < K$ }
		\STATE{$S_v \gets S_v \cup \{e\}$}
		\ENDIF
		\ENDFOR			
		\ENDFOR
		\STATE \textbf{return} $\arg\max_{S_v} f(S_v)$
	\end{algorithmic}
	\caption{SieveStreaming algorithm.}
	\label{fig:SieveStreaming-appendix}
\end{algorithm}

\begin{theorem}
  \label{th:SieveStreaming-appendix}
  SieveStreaming has the following properties \cite{Badanidiyuru/etal/2014}:
  \begin{itemize}
    \item It outputs a set $S$ such that $|S| \le K$ and $f(S) \ge \left(\frac{1}{2} - \varepsilon\right)OPT$
    \item It does $1$ pass over the data and stores at most $\mathcal O\left(\frac{K \cdot \log(K)}{\varepsilon}\right)$ elements
  \end{itemize}
\end{theorem}


\subsection{Salsa}

Norouzi-Fard et al. propose in \cite{Norouzi-Fard/etal/2018} a meta-algorithm for submodular function maximization called Salsa which uses different algorithms for maximization as sub-procedures. The authors argue, that there are different types of data-streams, namely `dense' and `sparse' streams. In dense streams, the utility value of the majority of items is equally large and thus we quickly find an appropriate item to add to the summary. In contrast, sparse streams only have a fraction of items with a high utility making it more difficult to find good items. For each stream type, a different thresholding-rule is appropriate. The authors use this intuition to design a $r$-pass algorithm that iterates $r$ times over the entire dataset and adapts the thresholds between each run. They show that their approach is a $1 - ( (r/(r+1))^r - \varepsilon)$ approximation algorithm. Note, that for a true streaming setting, i.e. $r = 1$, this algorithm recovers the $1/2 - \varepsilon$ approximation bound. The original paper focuses on this $r$-pass algorithm in which $OPT$ is known beforehand. However, the authors also present a streaming version of Salsa where $OPT$ is unknown in an extended version of the paper \cite{Norouzi-Fard/etal/2018b} (Appendix E, algorithm 7). This algorithm combines the ideas of the $r$-pass algorithm which uses optimized thresholds for each stream type with the idea of SieveStreaming to approximate $OPT$ with different sets of thresholds. The resulting algorithm runs multiple maximization algorithms $\mathcal A$ in parallel, each using different threshold-rules $A(v)$ for a given threshold $v$. The algorithm is depicted in Algorithm \ref{fig:SalsaStreaming-appendix} and its theoretical properties are presented in Table \ref{tab:Alg-Comparison-appendix}. Additionally, the authors present a variant with the same theoretical properties, in which $m$ is unknown and estimated on the fly. Note, that some of these algorithms require additional hyperparameters and properties of the stream which must be given beforehand. More specifically, the algorithms must know the size of the data-stream beforehand to choose appropriate thresholding values. Since this might be unknown real-world use-case this algorithm might not be applicable in all scenarios. 

\begin{algorithm}
	\begin{algorithmic}[1]
	    \STATE $O \gets \{(1+\varepsilon)^i \mid i \in \mathbb{Z}, m \le (1+\varepsilon)^i \le K \cdot m\}$
	    \STATE $\mathcal A \gets \texttt{init\_all\_algorithms}(e_1,...,e_T)$
	    \FOR{$A \in \mathcal A \text{~and~} v \in O$}
	        \STATE $S_{A,v} \gets \emptyset$
	    \ENDFOR
		\FOR{next item $e$}
		    \FOR{$A \in \mathcal A \text{~and~} v \in O$}
		        \IF{$\Delta_f(e | S_{A,v}) > A(v)$}
		            \STATE{$S_{A,v} \gets S_{A,v} \cup \{e\}$}
	            \ENDIF
	        \ENDFOR
		\ENDFOR
		\STATE \textbf{return} $\arg\max_{S_{A,v}} f(S_{A,v})$
	\end{algorithmic}
	\caption{Salsa algorithm.}
	\label{fig:SalsaStreaming-appendix}
\end{algorithm}

\begin{theorem}
  \label{th:Salsa-appendix}
  Salsa has the following properties \cite{Norouzi-Fard/etal/2018b}:
  \begin{itemize}
    \item It outputs a set $S$ such that $|S| \le K$ and $f(S) \ge \left(\frac{1}{2} - \varepsilon\right)OPT$
    \item It does $1$ pass over the data and stores at most $\mathcal O\left(\frac{K \cdot \log(K)}{\varepsilon}\right)$ elements
  \end{itemize}
\end{theorem}


\subsection{SieveStreaming++}

Recently, Kazemi et al. proposed in \cite{Kazemi/etal/2019} an extension of the SieveStreaming algorithm called SieveStreaming++. As outlined in the previous sections, SieveStreaming relies on the accurate estimation of the interval $[m, Km]$ to sample thresholds accordingly. As new items arrive from the stream many of the smaller thresholds will accept most items and therefore quickly fill-up their respective sieves. The authors point out, that the currently best performing sieve $S_v = \arg\max_v \{f(S_v)\}$ offers a better lower bound for the function value and they propose to use $[\max_v\{f(S_v)\}, K\cdot m]$ as the interval for sampling thresholds. This results in a more dynamic algorithm, in which sieves are removed once they are outperformed by other sieves and new sieves are introduced to make use of the better estimation of $OPT$. Algorithm \ref{fig:SieveStreaming++-appendix} depicts the resulting algorithm if the maximum singleton value $m=\max_e\{f(\{e\})\}$ is known beforehand. Similar to SieveStreaming, Kazemi et al. also present a version which estimates $m$ on the fly with the same theoretical properties. SieveStreaming++ manages $\mathcal O(\log K/\varepsilon)$ sieves in parallel and therefore makes $\mathcal O(\log K/\varepsilon)$ function-queries per item. Note that even though the algorithm manages $\mathcal O(\log K/\varepsilon)$ sieves, it only requires $\mathcal O(K / \varepsilon)$ memory because there is only one full sieve with $K$ elements in each given moment and
the remaining ones each have less than $K$ elements.

\begin{algorithm}
	\begin{algorithmic}[1]
	    \STATE $O \gets \emptyset, \tau_{min} \gets 0, \texttt{LB} \gets 0$
		\FOR{next item $e$}
		\STATE $\tau_{min} \gets \frac{\max\{\texttt{LB}, m\}}{2K}$
		\STATE $O \gets \{v|v > \tau_{min} \forall v \in O\}$
		\FOR{$v \in \{(1+\varepsilon)^i | \tau_{min}/(1+\varepsilon) \le (1+\varepsilon)^i \le \Delta\}$}
		\IF{$v \notin O$}
		\STATE $O \gets O \cup \{v\}$, $S_v \gets \emptyset$
		\ENDIF
		\IF{$\Delta_f(e | S_v) \ge \tau$ and $|S_v| \le K$ }
		\STATE $S_v \gets S_v \cup \{e\}$ 
		\STATE $\texttt{LB}\gets \max \{LB, f(S_v)\}$
		\ENDIF
		\ENDFOR
		\ENDFOR
		\STATE \textbf{return} $\arg\max_{S_v} f(S_v)$
	\end{algorithmic}
	\caption{SieveStreaming++ algorithm.}
	\label{fig:SieveStreaming++-appendix}
\end{algorithm}

\begin{theorem}
  \label{th:Sieve-Streaming++-appendix}
  Sieve-Streaming++ has the following properties \cite{Kazemi/etal/2019}:
  \begin{itemize}
    \item It outputs a set $S$ such that $|S| \le K$ and $f(S) \ge \left(\frac{1}{2} - \varepsilon\right)OPT$
    \item It does $1$ pass over the data and stores at most $\mathcal O\left(\frac{K}{\varepsilon}\right)$ elements
  \end{itemize}
\end{theorem}

\subsection{QuickStream}
Kuhnle proposed the QuickStream algorithm in \cite{kuhnle/2021} which works under the assumption that a single function evaluation is very expensive. QuickStream buffers up to $c$ elements and only evaluates $f$ every $c$ elements. If the function value is increased by the $c$ elements, they are all added to the solution. Additionally, older examples are removed if there are more than $K$ items in the solution. Algorithm \ref{fig:QuickStream-appendix} depicts the resulting algorithm and Theorem \ref{th:QuickStream-appendix} shows its theoretical properties. Somewhat surprisingly, QuickStreams performance deteriorates if it stores more elements. This counter-intuitive behavior can be explained by the fact that it only performs $\lceil n/c \rceil + c$ function evaluations and thus buffering more elements leads to fewer function evaluations. In turn, QuickStream trades memory for fewer function evaluations. Consequently, it does not have a fine-grained over which items it can pick, because it operates on chunks of $c$ items at-once.

\begin{algorithm}
	\begin{algorithmic}[1]
	    \STATE $A \gets \emptyset, l \gets \lceil \log_2(1/(4\varepsilon)) \rceil + 3$
 		\FOR{next item $e$}
 		\STATE $C \gets C \cup \{e\}$
 		\IF{$|C| = c$}
 		    \IF{$f\left(A \cup C \right) - f(A) \ge f(A) / K$}
 		        \STATE{$A \gets A \cup C$}
 		    \ENDIF
 		    \IF{$|A| \ge 2 c l (K+1) \log_2(K)$}
 		        \STATE{$A \gets \{c l (K+1)\log_2(K) \text{most recently added to A}\}$}
 		    \ENDIF
 	        \STATE{$C \gets \emptyset$}
 		\ENDIF
 		\ENDFOR
 		\STATE{$A \gets \{c K \text{most recently added to A}\}$}
 		\STATE{$R \gets \text{random\_partitioning}(A,c,K)$} \COMMENT{Split $A$ randomly into at-most $c$ sets with at-most  $K$ elements.}
 		\STATE \textbf{return} $\arg\max_{S \in \mathcal R} f(S)$ 
	\end{algorithmic}
	\caption{QuickStream algorithm.}
	\label{fig:QuickStream-appendix}
\end{algorithm}

\begin{theorem}
  \label{th:QuickStream-appendix}
  QuickStream has the following properties for $c\ge 1, \varepsilon \ge 0$ and $K \ge 2$  \cite{kuhnle/2021}:
  \begin{itemize}
    \item It outputs a set $S$ such that $|S| \le K$ and $f(S) \ge (1/(4c) - \varepsilon) OPT$
    \item It does $1$ pass over the data and stores at-most $\mathcal O\left(c K \log(K) \log(1/\varepsilon) \right)$ elements
  \end{itemize}
\end{theorem}


\subsection{Examples of Submodular Functions}
\label{sec:submodular-examples}
Until now, we have characterized the utility function as being non-negative, monotone, and submodular, but we left open, which particular function would be well suited for a data summary on the fly. For a data summary, a diverse sets of items is desired to fully capture the data stream. In other words, we want the observations in $S$ to be most dissimilar to each other. One way to write this more formally uses a kernel function $k(e_i, e_j)$ which expresses the similarity between two items $e_i,  e_j \in S$.
A common example for such a kernel function is the RBF kernel 
\begin{equation}
\label{eq:RBF-appendix}
k(e_i, e_j) = \exp\left(-\frac{1}{2l^2}\cdot ||e_i - e_j||^2_2\right)
\end{equation}
where $l\in\mathbb R$ is a scaling constant and $||\cdot||_2$ denotes the Euclidean norm. A kernel function gives rise to the kernel matrix $\Sigma_S = [k(e_i,e_j)]_{ij}$ containing all similarity pairs of all points in $S$. This matrix has the largest values on the diagonal as items are the most similar to themselves, whereas values on the off-diagonal indicate the similarity between distinct elements and thus are usually smaller. Since we seek a comprehensive summary, we are more interested in the pairs with values near $0$ on the off-diagonal of $\Sigma_S$. This intuition has been formally
handled in the context of the Informative Vector Machine (IVM) 
\cite{lawrence/etal/2003}. The IVM is a Gaussian Process (GP) \cite{Williams/Rasmussen/2006}, which greedily selects a subset of data points and keeps track of the GPs posterior distribution. Based on a diversity argument, the authors propose to iteratively select that point, which covers the training data best. This intuition can be formalized in maximizing the logarithmic determinant of the kernel matrix:
\begin{equation}
  \label{eq:logdet-appendix}
  f(S) = \frac{1}{2}\log\det(\mathcal I + a \Sigma_S)
\end{equation}
where $a\in\mathbb R_+$ is a scaling parameter for numerical robustness and $\mathcal I$ is the identity matrix.

In \cite{seeger/2004}, this function is shown to be submodular. Its function value does not depend on $V$, but only on the choice of the kernel function $k$ and the summary size $K$. This makes it an ideal candidate for summarizing data in a streaming setting. In \cite{Buschjaeger/etal/2017}, it is proven that $m = max_{e \in V}f(\{e\}) = 1 + a K$ and that $OPT \le K\log(1+a)$ for kernels with $k(\cdot, \cdot) \le 1$. This property can be enforced for every positive definite kernel with normalization \cite{graf/borer/2001}.  
A very similar function arises in the context of Determinantal Point Processes (DPP), which model a distribution over maximally diverse sets \cite{Kulesza/Taskar/2012}. Formally, a point process $\mathcal P(S)$ is a probability measure with $\mathcal P(S) \sim \det(\Sigma)$, where $\Sigma_S$ is the covariance matrix for $S$. It follows that a maximum a posteriori (MAP) estimation for $\mathcal P(S)$ requires the maximization of $\det \Sigma_S$. Taking the logarithm of this results in Equation \ref{eq:logdet-appendix} which is submodular. Thus, submodular maximization enables efficient MAP estimation for DPP. 

Last, we note that it is always possible to transform a submodular function depending on a large ground-set $V$ to a function that is evaluated only on some sample $W \subseteq V$. Badanidiyuru et al. give an approximation analysis in \cite{Badanidiyuru/etal/2014} of randomly sampling $W$ by using Hoeffding's inequality which shows that the approximation error between $f_V$ and $f_W$ reduces with $\mathcal O(|W|\log |W|)$. A more evolved discussion of this approach is given by Goemans et al. in \cite{goemans/etal/2009} which results in the same error rate by using an ellipsoid approximation for any submodular function $f$. Iyer et al. indicate that the approximation error can be further improved for certain classes of submodular functions \cite{iyer/etal/2013} by introducing the curvature $\kappa$ of $f$:
$$
\kappa = 1 - \min_{e\in V} \frac{\Delta_f(e|V\{e\})}{f(\{e\})}
$$
Intuitively the curvature of a submodular function measures its modularity, which is the degree of $f$ in which it can be approximated by the sum of simpler submodular functions: $f(S) \approx \sum_i f_i(S)$. The authors show that given a curvature $\kappa$ the approximation ratio of a simpler function $f_W$ depending only on $W$ instead of $V$ is $\mathcal O\left(\frac{\sqrt{|W|}\log |W|}{1+(\sqrt{|W|}\log(|W|)-1)(1-\kappa)}\right)$. Thus, every submodular function depending on the entire ground-set $V$ can be sufficiently approximated with a sample $W \subseteq V$. 

\section{The Three Sieves Algorithm}
\label{sec:three-sieves-appendix}

We recognize, that SieveStreaming and its extensions offer a worst-case guarantee on their performance and indeed they can be consider optimal providing an approximation guarantee of $\frac{1}{2} - \varepsilon$ under polynomial memory constraints \cite{feldman/etal/2020}. However, we also note that this worst case often includes pathological cases, whereas practical applications are usually much more well-behaved. One common practical assumption is, that the data is generated by the same source and thus follows the same distribution (e.g. in a given time frame). In this paper, we want to study these better behaving cases more carefully and present an algorithm which improves the approximation guarantee, while reducing memory and runtime costs in these cases. More formally, we will now assume that the items in the given sample (batch processing) or in the data stream (stream processing) are independent and identically distributed (iid). Note, that we do \textit{not} assume any specific distribution. For batch processing this means, that all items in the data should come from the same (but unknown) distribution and that items should not influence each other. From a data-streams perspective this assumptions means, that the data source will produce items from the same distribution which does not change over time. Hence, we specifically \textit{ignore} concept drift and assume that an appropriate concept drift detection mechanism is in place, so that summaries are e.g. re-selected periodically. We will study streams with drift in more detail in our experimental evaluation. 
We now use this assumption to derive an algorithm with $(1-\varepsilon)(1-1/\exp(1))$ approximation guarantee in high probability:

SieveStreaming and its extension, both, manage $\mathcal O(\log K / \varepsilon)$ sieves in parallel, which quickly becomes unmanageable even for moderate choices of $K$ and $\varepsilon$. We note the following behavior of both algorithms: Many sieves in SieveStreaming have \emph{too small a novelty-threshold} and quickly fill-up with uninteresting events. SieveStreaming++ exploits this insight by removing small thresholds early and thus by focusing on the most promising sieves in the stream. On the other hand, both algorithms manage sieves with \emph{too large a novelty-threshold}, so that they never include any item. Thus, there are only a few thresholds that produce valuable summaries. We exploit this insight with the following approach: Instead of using many sieves with different thresholds we use only a single summary and carefully calibrate the threshold. To do so, we start with a large threshold that rejects most items, and then we gradually reduce this threshold until it accepts some - hopefully the most informative - items. 
The set $O = \{(1+\varepsilon)^i \mid i \in \mathbb{Z}, m \le (1+\varepsilon)^i \le K \cdot m\}$ offers a somewhat crude but sufficient approximation of $OPT$ (c.f. \cite{Badanidiyuru/etal/2014}). We start with the largest threshold in $O$ and decide for each item if we want to add it to the summary or not. If we do not add any of $T$ items (which will be discussed later) to $S$ we may lower the threshold to the next smallest value in $O$ and repeat the process until $S$ is full. 

The key question now becomes: How to choose $T$ appropriately? If $T$ is too small, we will quickly lower the threshold and fill up the summary before any interesting item arrive that would have exceeded the original threshold. If $T$ is too large, we may reject interesting items from the stream. Certainly, we cannot determine with absolute certainty when to lower a threshold without knowing the rest of the data stream or knowing the ground set entirely, but we can do so with high probability. More formally, we aim at estimating the probability $p(e | f, S, v)$ of finding an item $e$ which exceeds the novelty threshold $v$ for a given summary $S$ and function $f$. Once $p$ drops below a user-defined certainty margin $\tau$
$$
p(e | f, S, v) \le \tau
$$
we can safely lower the threshold. This probability must be estimated on the fly.
Most of the time, we reject $e$ so that $S$ and $f(S)$ are unchanged and we keep estimating $p(e | f, S, v)$ based on the negative outcome. If, however, $e$ exceeds the current novelty threshold we add it to $S$ and $f(S)$ changes. In this case, we do not have any estimates for the new summary and must start the estimation of $p(e | f, S, v)$ from scratch. 
Thus, with a growing number of rejected items $p(e | f, S, v)$ tends to become close to $0$ and the key question is how many observations do we need to determine -- with sufficient evidence -- that $p(e | f, S, v)$ will be $0$. 

The computation of confidence intervals for estimated probabilities is a well-known problem in statistics. For example, the confidence interval of binominal distributions can be approximated with normal distributions, Wilson score intervals, or Jeffreys interval. Unfortunately, these methods usually fail for probabilities near $0$  \cite{brown/etal/2001}. However, there exists a more direct way of computing a confidence interval for heavily one-sided binominal distribution with probabilities near zero and iid data \cite{Jovanovic/Levy/1997}. The probability of not adding one item in $T$ trials is: 
$$
\alpha = \left(1-p(e | f, S, v) \right)^T \Leftrightarrow \ln\left(\alpha\right) = T \ln \left(1-p(e | f, S, v) \right)
$$

A first order Taylor Approximation of $\ln(1-p(e | f, S, v) )$ reveals that \\ $\ln \left(1-p(e | f, S, v) \right) \approx -p(e | f, S, v) $
and thus $\ln\left(\alpha\right) \approx T (-p(e | f, S, v) )$ leading to:
\begin{equation}
    \label{eq:RuleOfThree-appendix}
  \frac{-\ln\left(\alpha\right)}{T} \approx p(e | f, S, v) \le \tau
\end{equation}

Therefore, the confidence interval of $p(e | f, S, v) $ after observing $T$ events is $\left[0, \frac{-\ln\left(\alpha\right)}{T}\right]$. The $95\%$ confidence interval of $p(e | f, S, v) $ is $\left[0, -\frac{\ln(0.05)}{T}\right]$ which is approximately $[0, 3/T]$ leading to the term ``Rule of Three'' for this estimate \cite{Jovanovic/Levy/1997}. For example, if we did not add any of $T = 1000$ items to the summary, then the probability of adding an item to the summary in the future is below $0.003$ given a $\alpha = 0.95$ confidence interval. We can use the Rule of Three to quantify the certainty that there is a very low probability for finding a novel item in the data stream after observing $T$ items. Note that we can either set $\alpha$, $\tau$ and use Eq. \ref{eq:RuleOfThree-appendix} to compute the appropriate $T$ value. Alternatively, we may directly specify $T$ as a user parameter instead of $\alpha$ and $\tau$, thereby effectively removing one hyperparameter. We call our algorithm ThreeSieves and it is depicted in Algorithm \ref{fig:ThreeSieves-appendix}. Its theoretical properties are presented in Theorem \ref{th:ThreeSieve-appendix}.  
 
\begin{algorithm}
	\begin{algorithmic}[1]
    \STATE $O \gets \{(1+\varepsilon)^i \mid i \in \mathbb{Z}, m \le (1+\varepsilon)^i \le K \cdot m\}$
    \STATE $v \gets \max(O);~O \gets O\setminus\{\max(O)\}$
    \STATE $S \gets \emptyset; ~t \gets 0$
    \FOR{next item $e$}
        \IF{$\Delta_f(e | S) \ge \frac{v/2 - f(S)}{K - |S|}$ and $|S| < K$ }
          \STATE{$S \gets S \cup \{e\};~t \gets 0$}
        \ELSE
          \STATE{$t \gets t + 1$}
          \IF{$t \ge T$}
            \STATE $v \gets \max(O);~O \gets O\setminus\{\max(O)\}; t\gets 0$
          \ENDIF
        \ENDIF
    \ENDFOR	
    \STATE \textbf{return} $S$
	\end{algorithmic}
	\caption{ThreeSieves algorithm.}
	\label{fig:ThreeSieves-appendix}
\end{algorithm}

\begin{theorem}
  \label{th:ThreeSieve-appendix}
  ThreeSieves has the following properties:
  \begin{itemize}
    \item Given a fixed groundset $V$ or an infinite data-stream in which each item is independent and identically distributed (iid) it outputs a set $S$ such that $|S| \le K$ and with probability $(1 - \alpha)^K$ it holds for a non-negative, monotone submodular function $f$: $f(S) \ge (1- \varepsilon) (1 - 1/\exp(1))OPT$ 
    \item It does $1$ pass over the data (streaming-capable) and stores at most $\mathcal O\left(K\right)$ elements
  \end{itemize}
\end{theorem}

\begin{proof}
ThreeSieves does one pass over the data and inspects each data item once. It manages only one summary $S$ that stores at most $K$ elements. Note, that the set $O$ does not need to be materialized in memory before running the algorithm, but we can compute the appropriate thresholds from $O$ on the fly once they are required in line $10$. Thus, ThreeSieves uses $\mathcal O(K)$ memory. 
  
The Greedy Algorithm selects that element with the largest marginal gain in each round. Suppose we know the marginal gains $v_i = \Delta(S|e_i)$ for each element $e_i\in V$ selected into $S$ before running the Greedy algorithm. Then we can simulate the Greedy algorithm by stacking $K$ copies of the dataset consecutively and by comparing the gain $\Delta_f(e|S)$ of each element $e\in V$ against the respective marginal gain $v_i$. Let $v_1,\dots, v_K$ be the series of these gains. 

Now, consider ThreeSieves. Let $O$ be its set of thresholds. Assume that $v^*_1,\dots,v^*_K \in O$. Let $v^*_1$ denote the first threshold used by ThreeSieves before any item has been added to $S$. Then by the statistical test of ThreeSieves it holds with probability $1-\alpha$ that
$$
P(v_1 \not= v^*_1) \le \frac{-\ln(\alpha)}{T} \Leftrightarrow P(v_1 = v^*_1) > 1 - \frac{-\ln(\alpha)}{T}
$$
Consequently, it holds with probability $(1-\alpha)^K$ 
$$
P\left(v_1 = v^*_1,\dots,v_K = v^*_K \right) > \left(1-\frac{-\ln(\alpha)}{T}\right)^K
$$

So far we assumed that $O$ contains the thresholds that the Greedy algorithm would use during optimization. Knowing these thresholds is part of the problem we have to solve, making it a kind of a chicken and egg problem. First, ThreeSieves will always choose a smaller or equally large threshold compared to Greedy, because Greedy always picks the element with the largest gain. Second, by the construction of $O$ it holds that $(1-\varepsilon) v^*_i \le v_i \le v^*_i$, which follows from submodularity as well as the fact that we lower the thresholds in the comparison ``$\Delta_f(e | S) \ge \frac{v/2 - f(S)}{K - |S|}$'' (the formal proof for this statement can be found in \cite{Badanidiyuru/etal/2014}, section 5.2).

Now, let $e_K$ be the element that is selected by ThreeSieves after $K-1$ items have already been selected. Then, the function value $f(S_K)$ is a function of marginal gains:
$$
f(S_K) = f(S_{K-1} \cup \{e_K\}) = f(S_{K-1}) - \Delta(e_k|S_{K-1})
$$
Let $S_0 = \emptyset$ and recall that by definition $f(\emptyset) = 0$, then it holds with probability $(1-\alpha)^K$:
\begin{align*}
f(S_K)  &= f(\emptyset) + \sum_{i=1}^K \Delta(e_i|S_{K-1}) = \sum_{i=1}^K v_i
        \ge \sum_{i=1}^K \left(1-\varepsilon\right) v^*_{i} = \left(1-\varepsilon\right) \sum_{i=1}^K v^*_{i} = \left(1-\varepsilon\right) f_{G}(S_K) \\
        &\ge \left(1-\varepsilon\right) \left(1-1/\exp(1)\right) OPT 
        = \left(1 - 1/\exp(1) - \varepsilon + \varepsilon/\exp(1) \right) 
        > \left(\frac{1}{2} - \varepsilon \right) OPT 
\end{align*}
where $f_{G}$ denotes the solution of the Greedy algorithm. 
\end{proof}

Similar to SieveStreaming, ThreeSieves tries different thresholds until it finds one that fits best for the current summary $S$, the data $V$, and the function $f$. In contrast, however, ThreeSieves is optimized towards minimal memory consumption by maintaining one threshold and one summary at a time. If more memory is available, one may improve the performance of ThreeSieves by running multiple instances of ThreeSieves in parallel on different sets of thresholds. So far, we assumed that we know the maximum singleton value $m = max_{e\in V}f(\{e\})$ beforehand. If this value is unknown before running the algorithm we can estimate it on-the-fly without changing the theoretical guarantees of ThreeSieves. As soon as a new item arrives with a new $m_{new} > m_{old}$ we re-set the current summary and use the new upper bound $K \cdot m_{new}$ as the starting threshold. It is easy to see that this does not affect the theoretical performance of ThreeSieves: Assume that a new item arrives with a new maximum single value $m_{new}$. Then, all items in the current summary have a smaller singleton value $m_{old} < m_{new}$. The current summary has been selected based on the assumption that $m_{old}$ was the largest possible value, which was invalidated as soon as $m_{new}$ arrived. Thus, the probability estimate that the first item in the summary would be `out-valued' later in the stream was wrong since we just observed that it is being out-valued. To re-validate the original assumption we delete the current summary entirely and re-start the summary selection. 

\section{Experimental Evaluation}
\label{sec:experiments-three-sieves-appendix}
No changes. 

\section{Summary Selection in Astrophysics}
\label{sec:fact}
In this section, we want to highlight a real-world application of summarizing massive astrophysical data in the context of the FACT telescope \cite{Anderhub/etal/2013a}. 
Astrophysics studies celestial objects of several of hundred billion light-years away from the earth by observing the energy beams emitted by these sources using telescopes such as FACT (The First G-APD Cherenkov Telescope) \cite{anderhub/etal/2011}. 
Figure \ref{fig:shower} shows an air shower triggered by some cosmic ray beam, emitting Cherenkov light that is captured by the FACT telescope. 
Green indicates the telescopes surface, whereas blue indicates the amount of light hitting the sensors. Red indicates padding pixels that are used to form quadratic images from the telescope (discussed in more detail later).
The FACT telescope records roughly $60$ events per second, where each event amounts up to $3$ MB of raw data, resulting in a data rate of about $180$ MB/s \cite{Buschjaeger/etal/2020}. 
While physicists closely monitor the behavior of the telescope they cannot possibly review $60$ events per second in real-time. Thus, we propose to provide them with a data summary of e.g. last night's events so that she can review the summary. 
Given an interesting event $e_i$ in the summary, we then present similar events to the physicist by computing the similarity between $e_i$ and the remaining events $e \in V$. More formally, each event $e\in V$ is assigned to a reference point $e_i \in S$ with the largest similarity $k(e_i, e)$ in the summary. Upon request, we then present all events assigned to the reference point $e_i$ for further inspection. This way, the physicist can quickly get a general grasp of the last measurements without looking at all events. 
We want to highlight this approach on some real-world data more carefully. The goal of this experiment is to extract meaningful summaries from real-world data and present them to a domain expert who assesses the usefulness of the summary. 
Due to space reasons we only discuss one summary in this paper. Please, note, however, that the presented method is currently being deployed to automatically extract summaries and present them to domain experts in a dashboard for easy browsing and inspection.
For this experiment, we take the publicly available Crab Nebula observation data in which the FACT telescope was directed at the well known Crab Nebula \cite{Biland/etal/2014a}. The data\footnote{\url{https://fact-project.org/data/}} consist of 17.7 hours of total recording with $3,972,043$ recorded events. We focus on the data from 01-11-2013 which contains $676,331$ events. For the previous runtime experiments, we randomly sampled $200,000$ observations from this set, whereas in this experiment now we use all available observations. 

\begin{figure}[b!]
\centering
\begin{tikzpicture}[scale=0.8,transform shape]
\begin{scope}[scale=0.8]
\clip[] (-3.4,-2) rectangle (3.5,5.25);

\begin{scope}[shift={(-2.5,-1.0)},scale=0.5]
\draw (0.3,-0.35) -- (0.0,-1);
\draw (0.5,-0.35) -- (0.8,-1);

\begin{scope}[rotate=-20]
\draw (0.15,-0.1) -- (0.4,0.55);
\draw (0.85,-0.1) -- (0.6,0.55);

\node (CAM) at (0.5,0.6) {};

\draw[fill=black!10] (0.4,0.45) -- (0.6,0.45) -- (0.6,0.7) -- (0.4,0.7) -- (0.4,0.45);

\fill[fill=black!10] (0,0) to[out=300,in=240] (1,0) -- (0,0);
\draw (0,0) to[out=300,in=240] (1,0) -- (0,0);

\end{scope}
\end{scope}

\fill[fill=brown!80,draw=brown,opacity=0.25] (-5,-2) -- (5,-2) -- (5,-1.5) to[out=178,in=2] (-5,-1.5);

\fill[fill=blue!10,draw=blue!20,opacity=0.25] (-5,1.75) to[out=5,in=175] (5,1.75) -- (5,4) to[out=175,in=5] (-5,4);

\draw[thick] (-0.05,4.26) -- (0.25,5);
\node at (0.5,5) {$\gamma$};

\node at (2,3) {\color{black!60}{\footnotesize{{Atmosphere}}}};

\node at (-2,3.25) {\color{black!60}{\footnotesize{{Air Shower}}}};

\node at (0.35,0.2) {\color{black!60}{\footnotesize{{Cherenkov Light}}}};

\node[anchor=west] at (-1.8,-1.1) {\color{black!60}{\footnotesize{{Telescope}}}};

\begin{scope}[shift={(0.1,4.6)},scale=0.2,rotate=160]
\input{figures/shower}
\end{scope}

\fill[fill=blue,opacity=0.1,rotate=-10] (-1.375,1.5) ellipse (0.5 and 0.25);
\fill[fill=blue,opacity=0.15,rotate=-10] (-1.5,1) ellipse (0.5 and 0.25);
\fill[fill=blue,opacity=0.2,rotate=-10] (-1.625,0.5) ellipse (0.5 and 0.25);
\fill[fill=blue,opacity=0.5,rotate=-10] (-1.75,0) ellipse (0.5 and 0.25);
\end{scope}

\node at (6,1.6) {
	\includegraphics[scale=0.05]{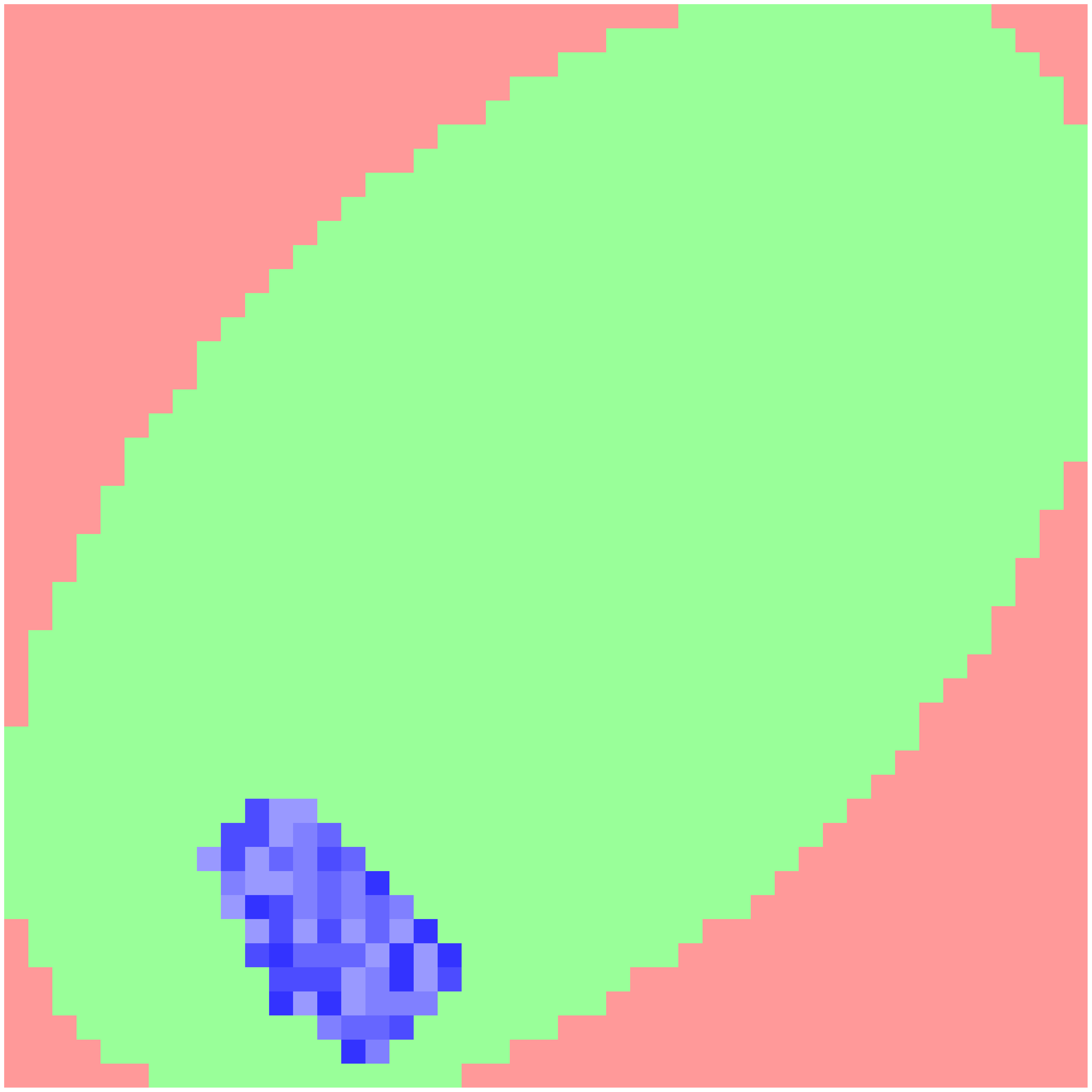}
};

\begin{scope}[shift={(0,-1)}]
\node at (6,0.35) {
	\includegraphics[scale=0.014]{figures/sensor_mapping.pdf}
};
\node at (6.75,0.35) {
	\includegraphics[scale=0.014]{figures/sensor_mapping.pdf}
};
\node at (7.5,0.35) {
	\includegraphics[scale=0.014]{figures/sensor_mapping.pdf}
};

\draw[thick,black!50] (5.65,0) rectangle (6.35,0.7);

\node at (5.25,0.35) {
	\includegraphics[scale=0.014]{figures/sensor_mapping.pdf}
};
\node at (4.5,0.35) {
	\includegraphics[scale=0.014]{figures/sensor_mapping.pdf}
};

\node at (3.75,0.35) {
	\includegraphics[scale=0.014]{figures/sensor_mapping.pdf}
};
\end{scope}

\draw[thick,black!50] (5.65,-0.3) -- (4.5,0.4);
\draw[thick,black!50] (6.35,-0.3) -- (7.5,0.4);

\draw[thick,black!50,->] (CAM) -- +(5,0);

\node at (6,-1.45) {\color{black!60}{\footnotesize{{Camera Samples (2000 MHz)}}}};

\end{tikzpicture}
\caption{\label{fig:shower}An air shower produced by a particle beam hitting the atmosphere (left) and the FACT telescope which measures it (right side). Picture was taken from \cite{Buschjaeger/etal/2020} with courtesy of the authors.}
\end{figure}

For the FACT telescope data, a long pre-processing pipeline has been developed to calibrate and clean the data and extract meaningful high-level features \cite{bockermann/etal/2015} (previously denoted as FACT-Highlevel). While this approach offers some form of interpretability due to the high-level features, it has recently been outperformed by Deep Learning methods in terms of energy reconstruction error and gamma/hadron separation performance. Buschjäger et al. presented in \cite{Buschjaeger/etal/2020} a processing pipeline for the FACT data that trains and deploys Convolutional Neural Networks for gamma/hadron classification. We leverage their insights to train autoencoders on the raw data that embed it into a lower-dimensional space. Then, we perform the data summarization using the lower-dimensional features from the autoencoder. To the best of our knowledge, this is the first attempt to use autoencoders for this type of astrophysical data. 
The data is publicly available and our code is part of this submission. We will make it publicly available upon publication.

The authors kindly provided us with their processing pipeline which is explained in detail in \cite{Buschjaeger/etal/2020}. It works as follows: FACT can be viewed as a camera with $1440$ pixels arranged in hexagonal form. Once enough energy hits the telescope's surface to trigger recording, FACT produces a time series of $300$ nanoseconds for each pixel. We extract the number of photons arriving at each pixel in the time series by subtracting calibration measurements from the time series as often as possible until there is no signal left \cite{Mueller/etal/2017a}. The number of subtractions can be considered the number of photons that arrived during the time series. 
Figure \ref{fig:shower} shows FACT with an air shower and the resulting data. In hexagonal grids, each pixel has up to six neighbors instead of four as in regular Euclidean grids. To apply regular CNNs with Euclidean filters the data is transformed into $45 \times 45$ images in which the hexagonal grid is slightly rotated into the middle of the image. Figure \ref{fig:shower} (right side) depicts the sensor mapping. Here, the red color denotes unused pixels which are always `0`, whereas green (and blue) pixels are mapped to the corresponding sensors.  Buschjäger et al. report that previous works did not detect any performance difference between using rectangular and hexagonal filters for FACT. 

To embed the raw measurements of FACT into a lower-dimensional feature space we trained a VGG-like autoencoder architecture with an encoder using 2 convolutional layers with $16$ filters followed by 1 linear layer with $8096$ neurons and a final embedding layer. The decoder mirrors the encoder in reversed order. Each convolution layer was followed by a batch-norm layer. Between the last batch-norm layer and the linear layer max-pooling with size and stride 2 was used. As activation function we used ReLu. The network was trained to minimize the reconstruction loss in terms of the mean-squared error (MSE) over $50$ epochs with batch size $128$ using the Adam Optimizer in PyTorch \cite{Paszke/etal/2019}. We experimented with different embedding sizes $\{32,64,\dots,512\}$ and found $256$ to be a good trade-off between reconstruction error and overall run-time. The average absolute reconstruction error per pixel was $0.2$, i.e. on average $1$ photon every $5$ pixel was wrongly reconstructed. For summary extraction, we used ThreeSieves with $T = 5000$ and $\varepsilon = 0.005$. As above, we maximize the log-determinant with RBF kernel and with $l = \frac{1}{2\sqrt{0.5\cdot d}}$ and $d = 256$.

\begin{figure*}[b!]
\centering
\captionsetup[subfigure]{labelformat=empty}
    \subfloat[(1)]{
      \includegraphics[width=0.2\textwidth]{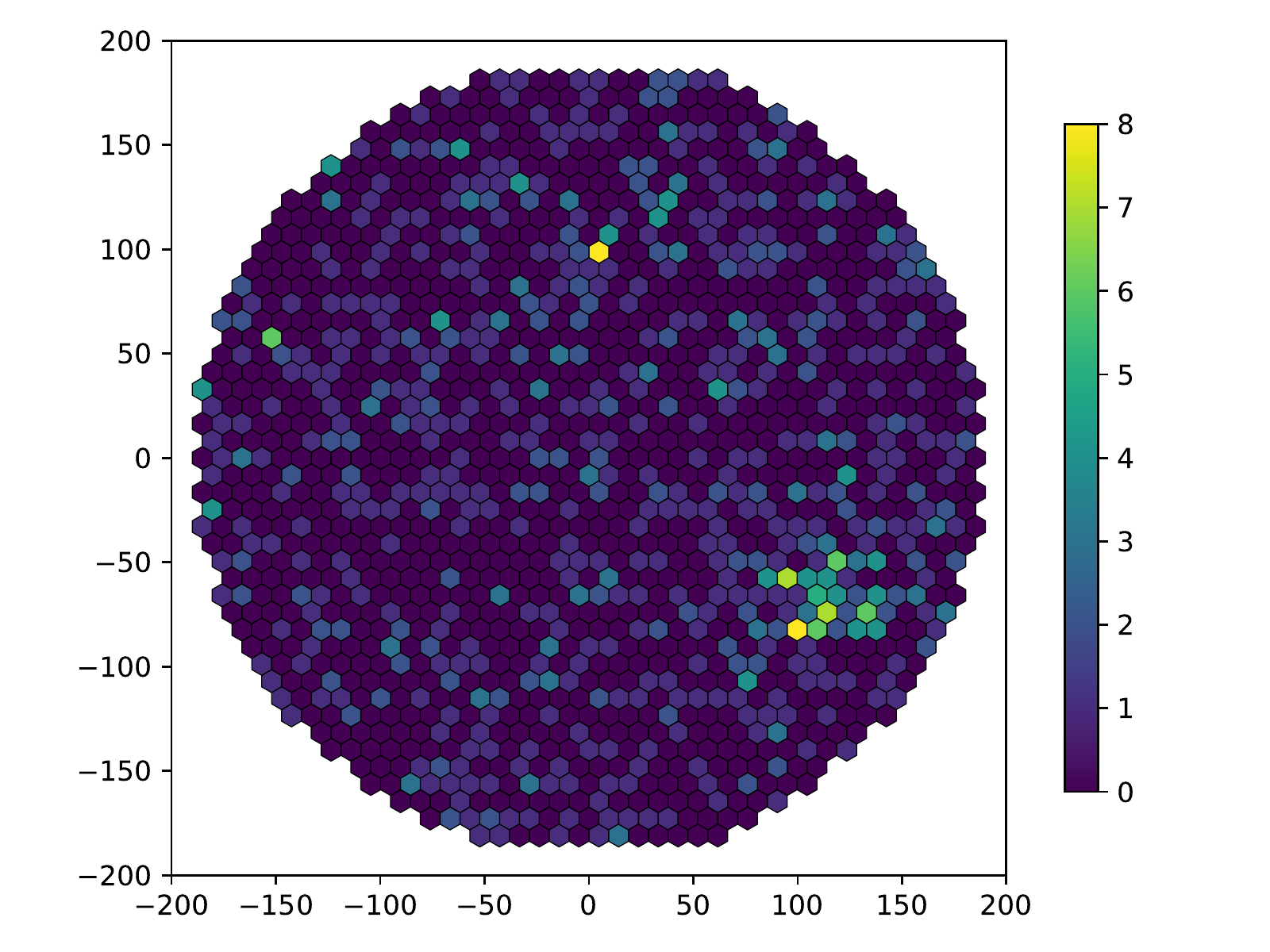}
    }
    \subfloat[(2)]{
      \includegraphics[width=0.2\textwidth]{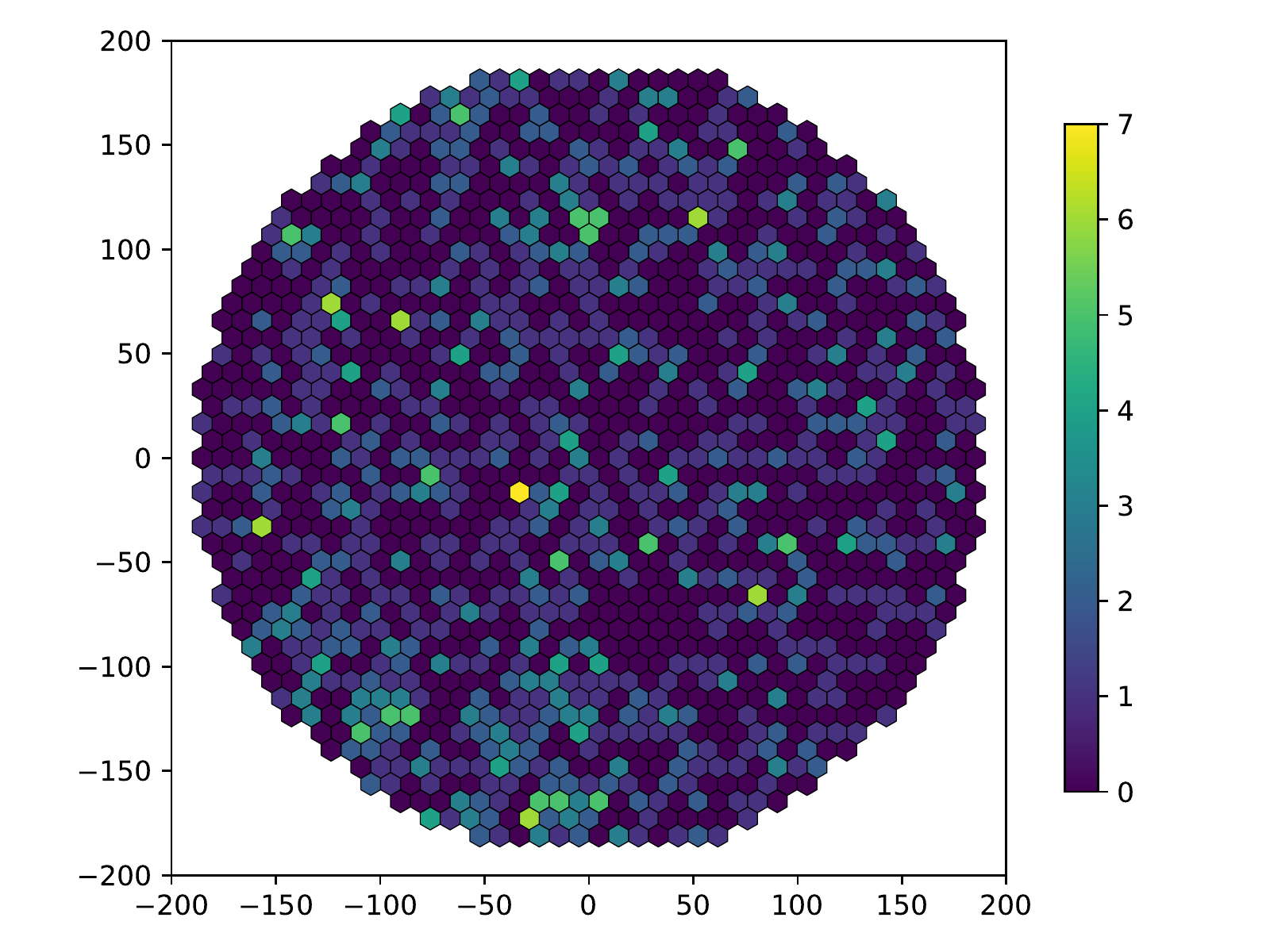}
    }
    \subfloat[(3)]{
      \includegraphics[width=0.2\textwidth]{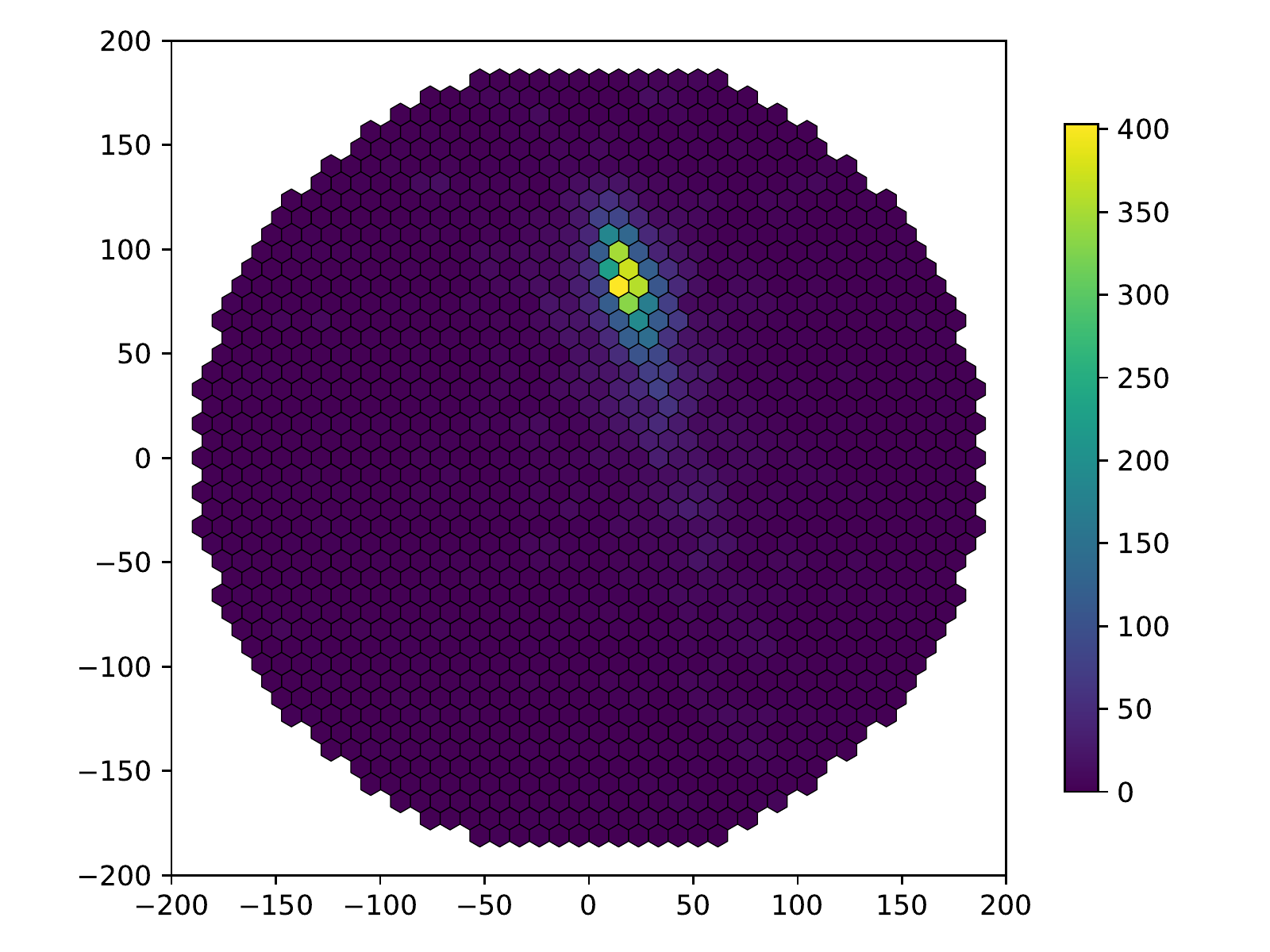}
    }
    \subfloat[(4)]{
      \includegraphics[width=0.2\textwidth]{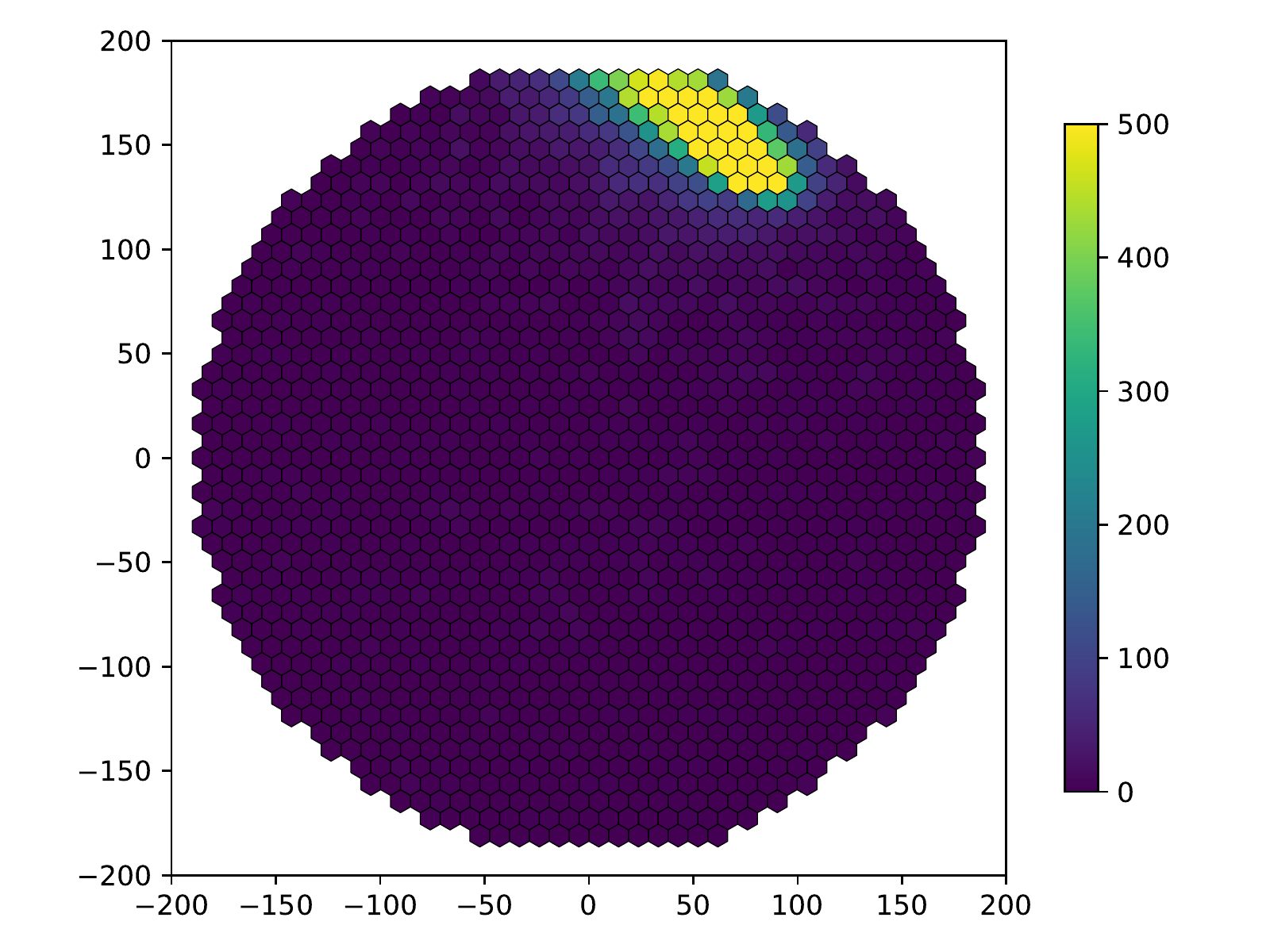}
    }
    \subfloat[(5)]{
      \includegraphics[width=0.2\textwidth]{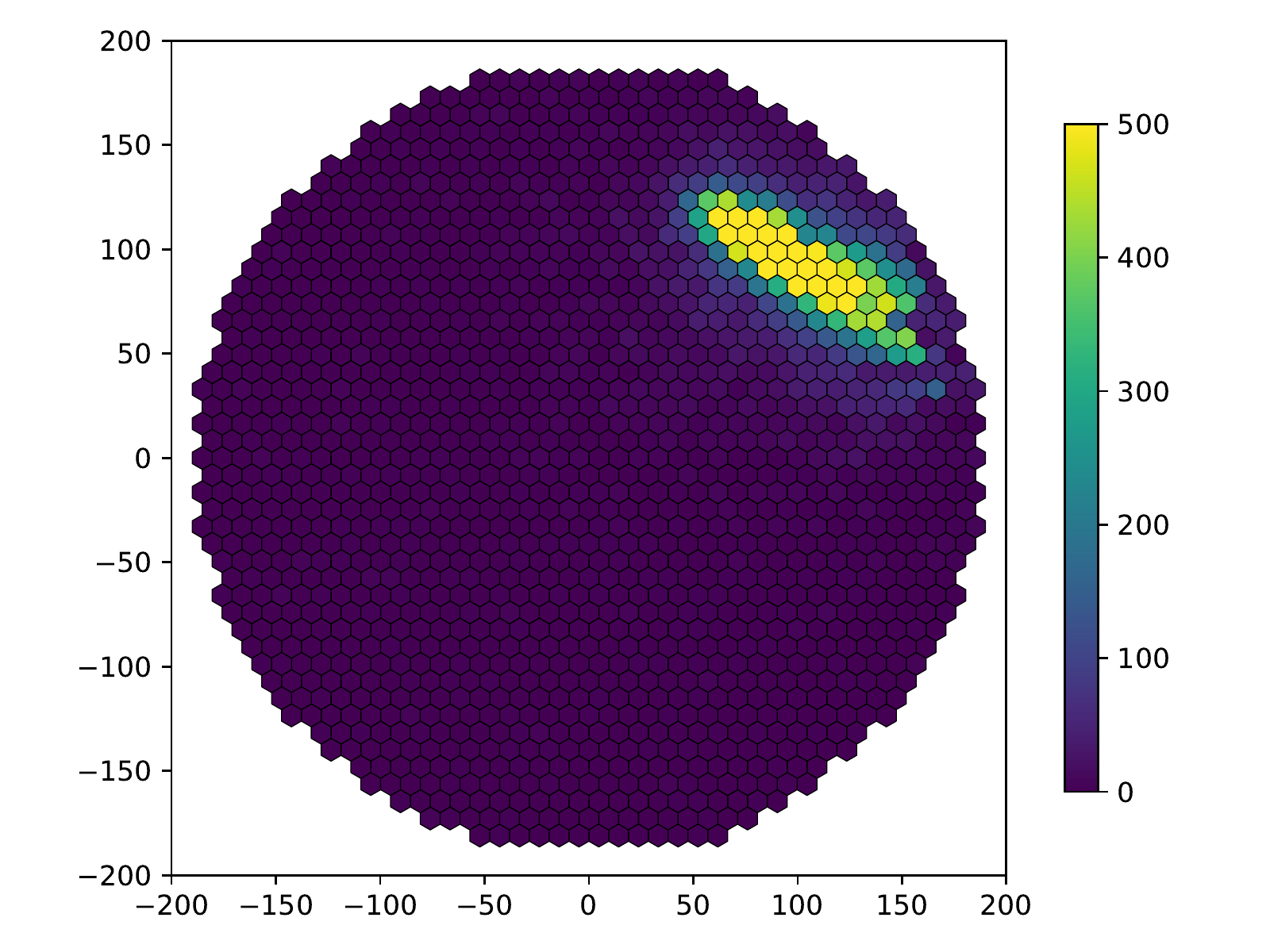}
    }
    \hspace{0mm}
    \subfloat[(6)]{
      \includegraphics[width=0.2\textwidth]{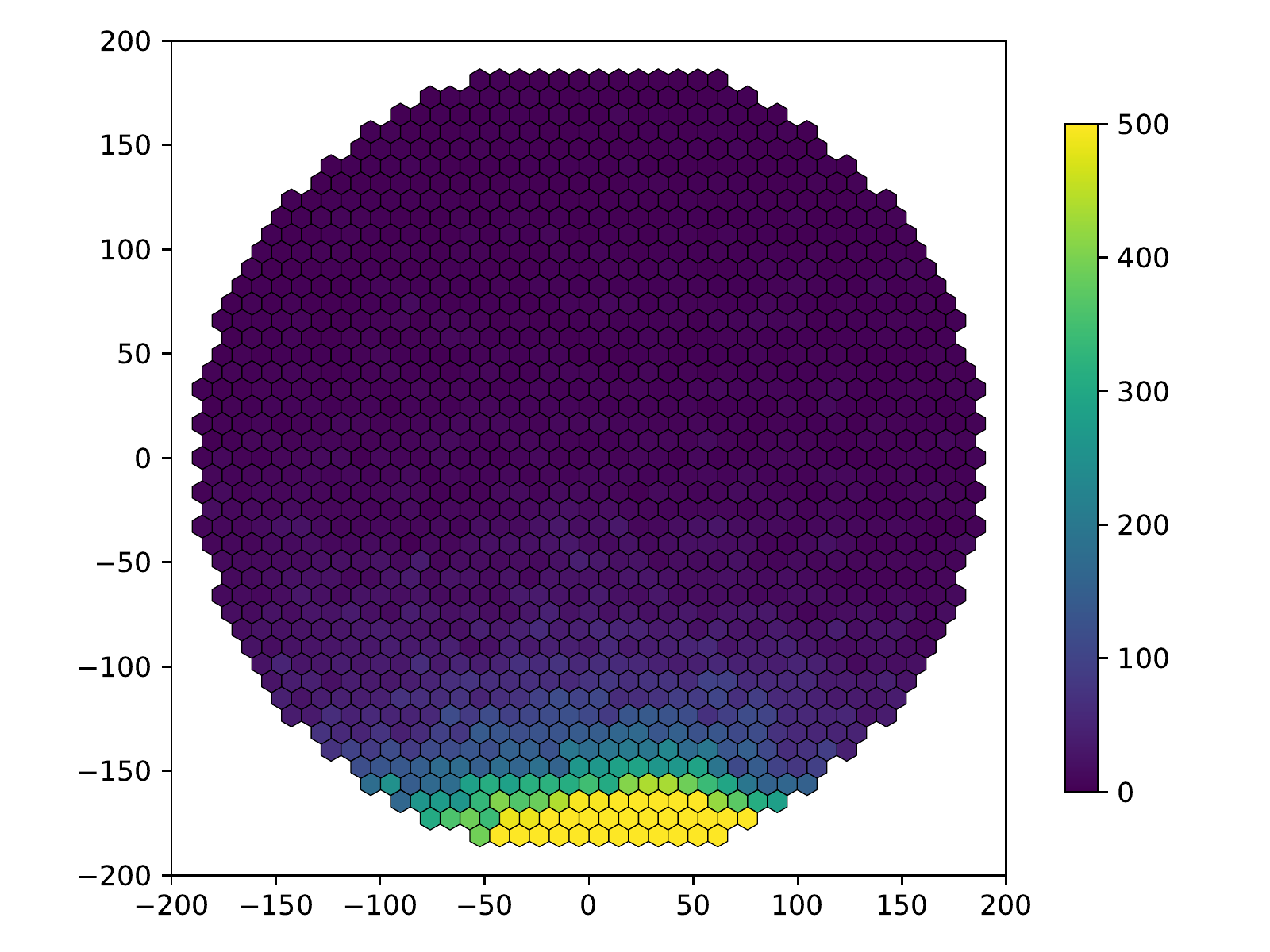}
    }
    \subfloat[(7)]{
      \includegraphics[width=0.2\textwidth]{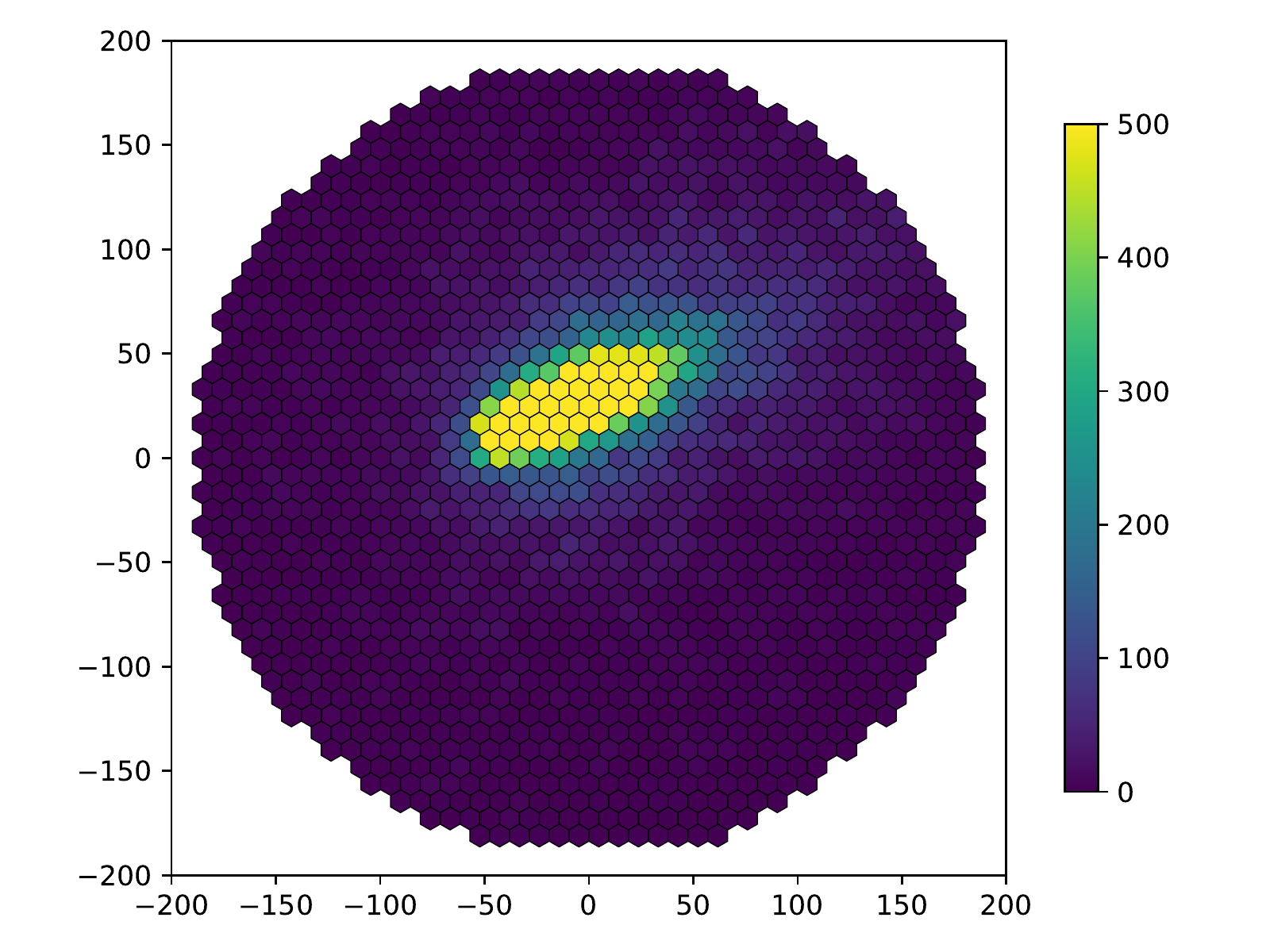}
    }
    \subfloat[(8)]{
      \includegraphics[width=0.2\textwidth]{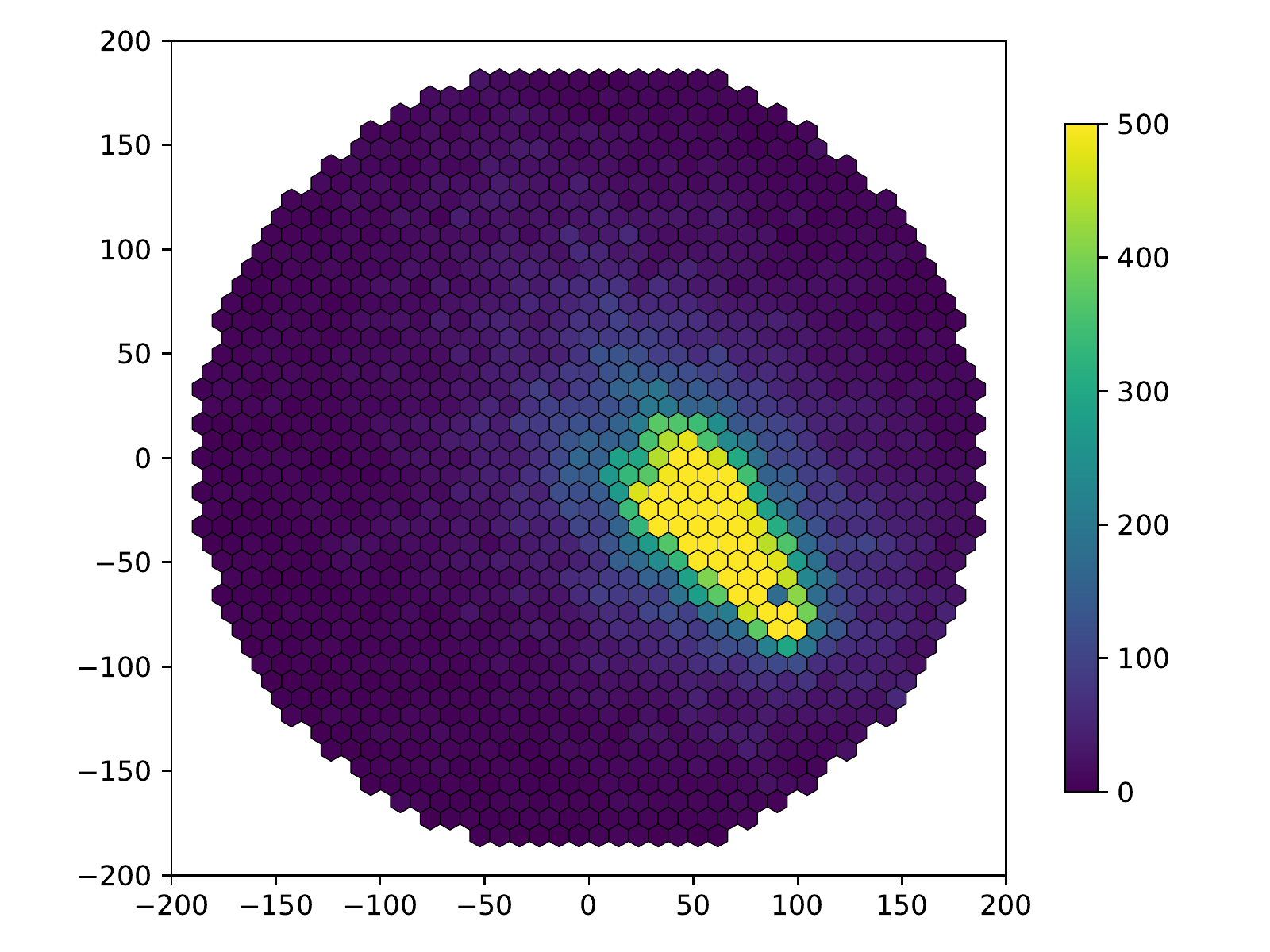}
    }
    \subfloat[(9)]{
      \includegraphics[width=0.2\textwidth]{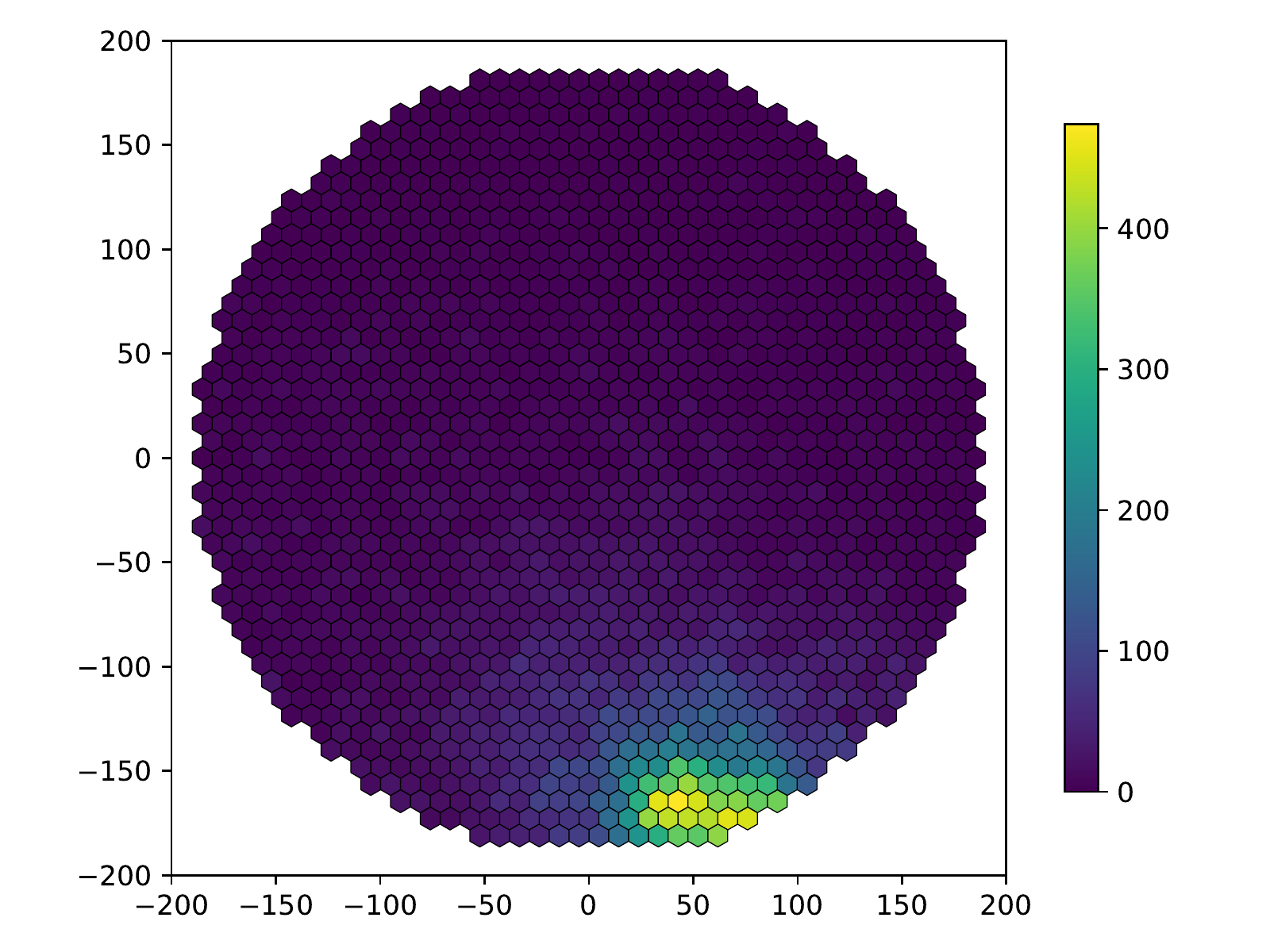}
    }
    \subfloat[(10)]{
      \includegraphics[width=0.2\textwidth]{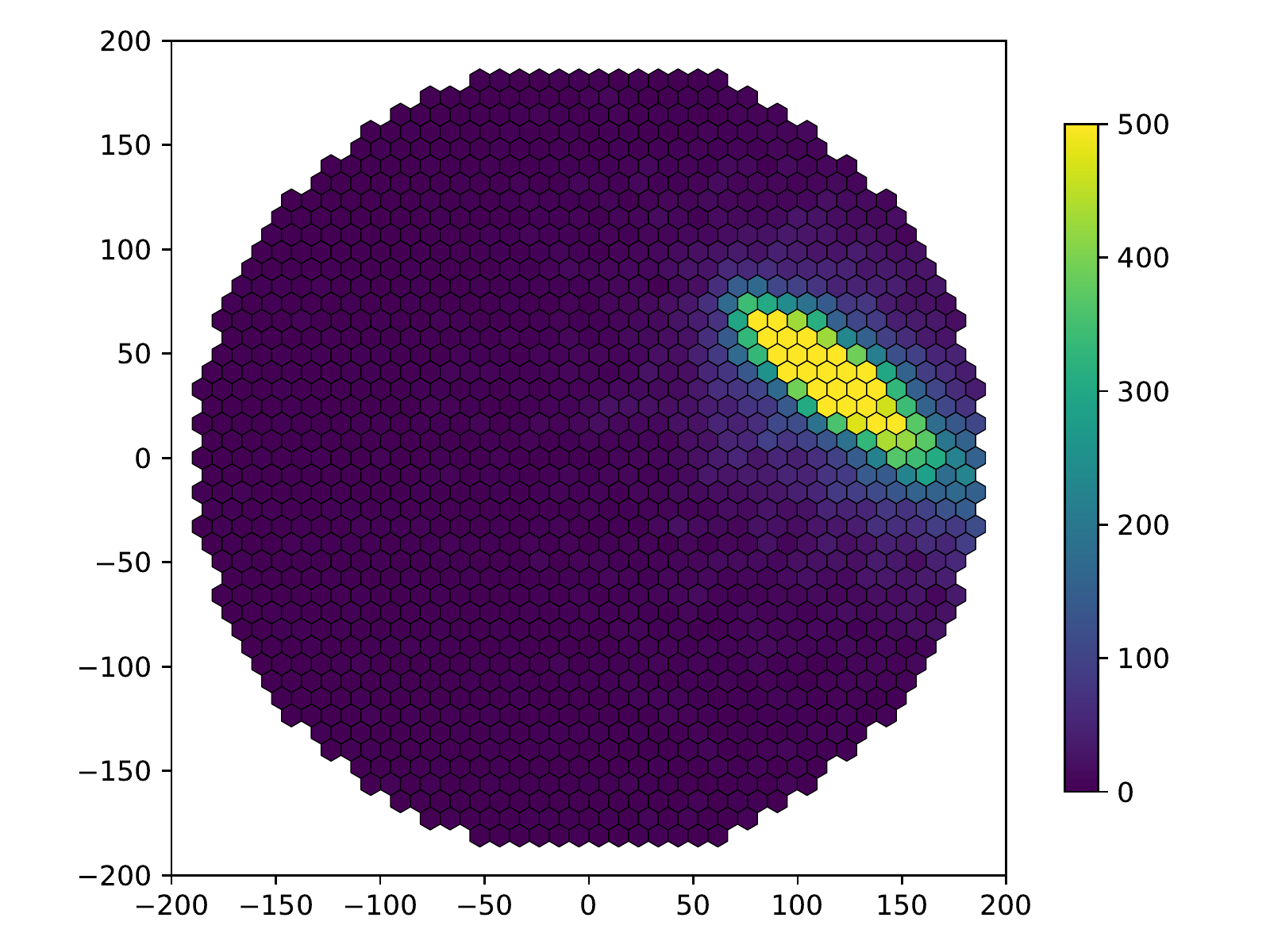}
    }
    \caption{Data summary of Crab Nebula data from 01-11-2013 computed with ThreeSieves using $\varepsilon = 0.005, T = 5000$.\label{fig:fact_summary}}
\end{figure*}

Figure \ref{fig:fact_summary} depicts the extracted summary. Each image shows the telescope's surface consisting of $1440$ sensors arranged in a hexagonal grid. Brighter colors highlight that more photons are hitting the respective sensors indicating a possible interesting event. 
A domain expert gave the following interpretation: Figure \ref{fig:fact_summary}-(2) depicts the night sky background where no actual event was measured. Figures \ref{fig:fact_summary}-1 and \ref{fig:fact_summary}-3 depict very small events with a dominant night sky background. In contrast, Figures \ref{fig:fact_summary}-5, \ref{fig:fact_summary}-8, and \ref{fig:fact_summary}-10 show a typical ellipsoid shape which indicates a high energy event. Here, both Figures \ref{fig:fact_summary}-1 and \ref{fig:fact_summary}-8 depict potential proton events due to the broader shape of the entire shower, whereas figures \ref{fig:fact_summary}-5 and \ref{fig:fact_summary}-10 clearly show gamma events.  Figures \ref{fig:fact_summary}-4, \ref{fig:fact_summary}-6, and \ref{fig:fact_summary}-9 depict so-called corner clippers, in which a shower was not completely recorded due to the position and orientation of the telescope. It is also interesting, that in figure \ref{fig:fact_summary}-8 seems to be a dead pixel (broken sensor) inside of the shower, as was indicated by the expert. 

\section{Conclusion}
\label{sec:conclusion-appendix}
No changes.

\end{document}